\documentclass[twoside,11pt]{article}

\usepackage{jmlr2e}
\usepackage{mlmodern}

\hypersetup{
  citecolor=blue,
  linkcolor=black,  
  colorlinks=true,
  urlcolor=blue
}

\usepackage{commath}
\usepackage{microtype} 
\usepackage{mathtools}
\usepackage{thmtools}
\usepackage{cleveref}
\usepackage{bm}

\usepackage{pgfplots}
\usepgfplotslibrary{groupplots}
\pgfplotsset{compat=1.18}

\usepackage{scrextend}

\usepackage{wrapfig}

\usepackage[edges]{forest}
\usepackage{makecell}

\usepackage{placeins}

\definecolor{bluegraph}{rgb}{0.0, 0.33, 0.71}

\numberwithin{equation}{section}
\numberwithin{theorem}{section} 

\newcommand{\qed}{\hfill\BlackBox\\[2mm]}

\renewcommand{\b}[1]{\mathsf{#1}}
\newcommand{\T}{\intercal}

\DeclarePairedDelimiterX\Set[2]{\lbrace}{\rbrace}%
{ #1 \,:\, #2 }                                         
\DeclarePairedDelimiterX\inprod[2]{\langle}{\rangle}%
{ #1 , #2 }                                             

\newcommand{\R}{\mathbb{R}} 
\newcommand{\N}{\mathbb{N}} 
\newcommand{\C}{\mathbb{C}} 

\newcommand{\GP}{\textup{GP}}
\newcommand{\V}{\mathbb{V}}

\newcommand{\fGP}{\mathsf{f}_\GP}

\jmlrheading{}{2026}{}{}{}{}{Toni Karvonen and Chris Oates}

\ShortHeadings{Why not to use the Gaussian kernel}{Karvonen and Oates}
\firstpageno{1}

\begin{document}

\title{Why not to use the Gaussian kernel}
\author{\name Toni Karvonen \email toni.karvonen@lut.fi \\
       \addr School of Engineering Sciences \\
       Lappeenranta--Lahti University of Technology LUT \\
       Yliopistonkatu 34, 53850 Lappeenranta, Finland
       \\ \\
  \name Chris J. Oates \email chris.oates@ncl.ac.uk \\
       \addr School of Mathematics, Statistics \& Physics \\
       Newcastle University \\
       Newcastle upon Tyne, NE1 7RU, United Kingdom}       
\editor{}

\maketitle

\begin{abstract}
  Kernels measure similarity or correlation in tasks such as regression and classification. 
  The Gaussian kernel, other names of which include squared exponential and radial basis function kernel, is one of the most popular in Gaussian process regression. 
  We argue that the Gaussian kernel is best avoided and should never be used as a default. 
  The argument rests on two results demonstrating that the Gaussian kernel is extremely brittle.
  First, the Gaussian kernel gives rise to a conditional variance that is unrealistically small.
  If the variance is used to quantify predictive uncertainty, catastrophic overconfidence is almost inevitable.
  Second, a small variance goes hand in hand with numerical ill-conditioning, so that to use the Gaussian kernel in practice requires tricks such as nugget terms that effectively modify the underlying regression or classification model. 
  These problems are caused by the unnatural smoothness of the Gaussian kernel, a fact we are far from the first to take notice of.
  The problem is not the Gaussian form itself but the analyticity of the kernel: Our argument is more broadly that analytic kernels are best avoided.  
  For stationary kernels analyticity is essentially equivalent to an exponential decay of the spectral density.  
\end{abstract}


\section{Introduction}

Kernels measure similarity or correlation between inputs and underpin many popular methods for tasks such as regression, classification, principal component analysis, and space-filling design.
Although a variety of needs have bred a great variety of kernels, some have always remained more popular than others.
The purpose of this article is to make a case against the persistent use of the Gaussian kernel in Gaussian process regression.
Let $\ell > 0$ be a lengthscale parameter.
In regression, the \emph{Gaussian kernel} (see Figure~\ref{fig:kernels-intro})
\begin{equation} \label{eq:gaussian-intro}
  k(x, y) = \exp\bigg( \! - \frac{ \norm[0]{x - y}_2^2}{2 \ell^2} \bigg)
\end{equation}
measures the correlation of inputs $x, y \in \R^d$.
Also known as squared exponential, exponentiated quadratic, and radial basis function kernel, the Gaussian is probably the most popular kernel for Gaussian processes and kernel methods in machine learning.
Its popularity goes back to the very introduction of Gaussian processes to machine learning~\citep{WilliamsRasmussen1996, Neal1998}.
The Gaussian kernel is \emph{infinitely smooth}, which is to say that one can differentiate it arbitrarily many times.
In fact, it is much more than infinitely smooth: It is \emph{analytic} with infinite radius of convergence, which means that its Taylor series developed at any point converges on the whole of $\R^d$.
It is well recognised that the Gaussian kernel is popular despite severe flaws it has:

\begin{quote}
  ``Stein~[1999] argues that such strong smoothness assumptions are unrealistic for modelling many physical processes, and recommends the Matérn class [...]. However, the squared exponential is probably the most widely-used kernel within the kernel machines field.'' --- \citet[p.\@~83]{RasmussenWilliams2006}
\end{quote}

\noindent \Cref{sec:apology} collects a number of reasons that we believe have contributed to the enduring popularity of the Gaussian kernel.

\begin{center}
  \begin{figure}[t]    
    \includegraphics[width=\textwidth]{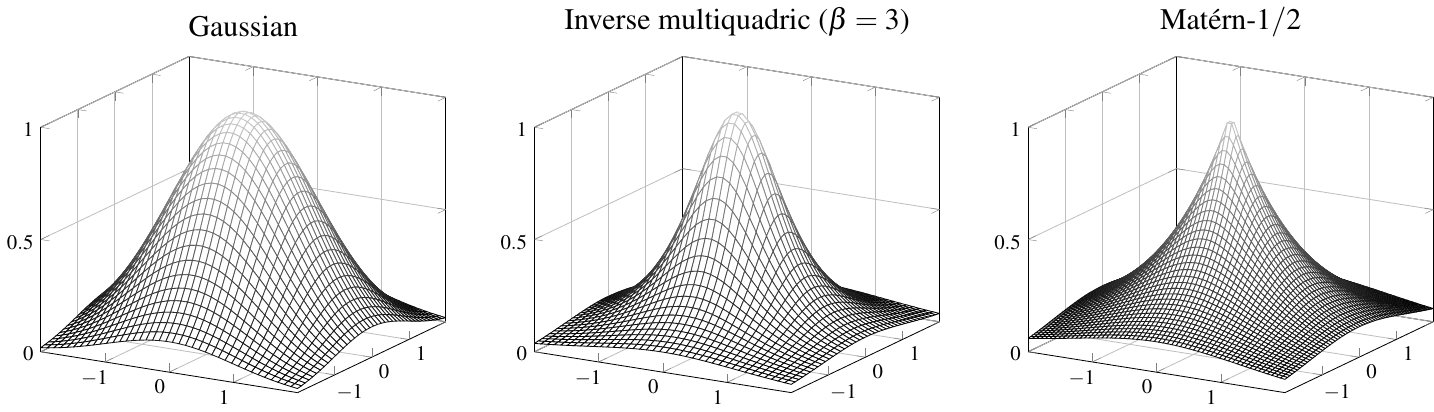}    
    \caption{Three kernels on $\R^2$ that appear in this article: The Gaussian kernel in~\eqref{eq:gaussian-intro}, the inverse multiquadric kernel in~\eqref{eq:multiquadric}, and a Matérn kernel in~\eqref{eq:matern} with $\nu = 1/2$. The plots depict kernel translates at the origin, $k(0, \cdot)$. All kernels use lengthscale $\ell = 1$.}
    \label{fig:kernels-intro}
  \end{figure}    
\end{center}

\subsection{Contributions}

Even though it has seen three decades of successful usage in machine learning (and much more in other fields), we argue that the Gaussian kernel should not be used in Gaussian process regression.
As is evident from the quote above and further citations below, we are far from the first to make this argument.
However, no comprehensive case against the Gaussian relying on mathematically rigorous results on its failure in practical scenarios has been made.
This article makes such a case.
We do not purport to demonstrate that the Gaussian kernel never works; there is ample evidence to show that it often does.
Rather, we argue that, in the absence of extremely strong prior information, one should never \emph{default} to the Gaussian kernel in applications.
At the time of writing, this recommendation is violated by many popular software packages, such as \texttt{GPy}~\citep{gpy2014} and \texttt{scikit-learn}~\citep{scikit-learn}. 
By using the Gaussian kernel one makes, knowingly or not, an analyticity assumption that the function being modelled is completely determined by its values in the neighbourhood of any given point.
This assumption is rarely, if ever, true.
While this alone might be (and has been) enough to convince some to shy away from the Gaussian kernel, the actual argument we make rests on two conspicuous implications of analyticity.

\paragraph{Problem 1: Overconfident uncertainty quantification.} Gaussian processes provide means for prediction at unseen input locations and quantification of predictive uncertainty via conditional mean and variance functions (see \Cref{sec:gps}).
The magnitude of the variance at a given point informs the user on the accuracy of the conditional mean. 
If the variance is small at a point $x$, the mean is understood to be a good approximation to the underlying latent function that generated the data.
Typically it is reasonable to expect the variance at $x$ to be small only if the dataset contains input points that lie close to $x$.
\emph{This is not what happens when the Gaussian kernel is used.}
Figure~\ref{fig:intro-variance} plots the conditional variance over $[-2, 2]^2 \subset \R^2$ for three different kernels when data are available only on the small sub-domain $[1, 2]^2 \subset \R^2$.
For the finitely smooth Matérn-$1/2$ and Matérn-$3/2$ kernels [see~\eqref{eq:matern}] the variance behaves as expected, increasing rapidly as one moves away from the sub-domain that contains the input points.
In contrast, for the Gaussian kernel the variance is much smaller on $[-2, 2]^2$ and essentially zero on $[0, 1]^2$, meaning that the model assigns little to no uncertainty for predictions made anywhere on the domain.
If the Gaussian kernel is used to quantify uncertainty in predictions when the domain of interest is not sufficiently covered by the input points, there is a real risk of overconfident uncertainty quantification that could prove disastrous in downstream tasks or in safety-critical applications.
The behaviour of the variance associated to the Gaussian kernel that is summarised in Figure~\ref{fig:intro-variance} is explained and quantified in Theorems~\ref{thm:interpolation-1d}, \ref{thm:interpolation}, and~\ref{thm:regression}.
Estimation of the lengthscale parameter $\ell$ can provide no remedy: Shrinking $\ell$ as more data are obtained slows down the rate of decay of the variance \emph{uniformly across the domain}. If, under a fixed lengthscale, the variance tends to zero everywhere even though all data are concentrated on a small sub-domain, it does so also when the lengthscale is estimated.
Remark~\ref{remark:parameter-estimation} expands upon hyperparameter estimation.

\begin{center}
  \begin{figure}[t]    
    \includegraphics[width=0.3\textwidth]{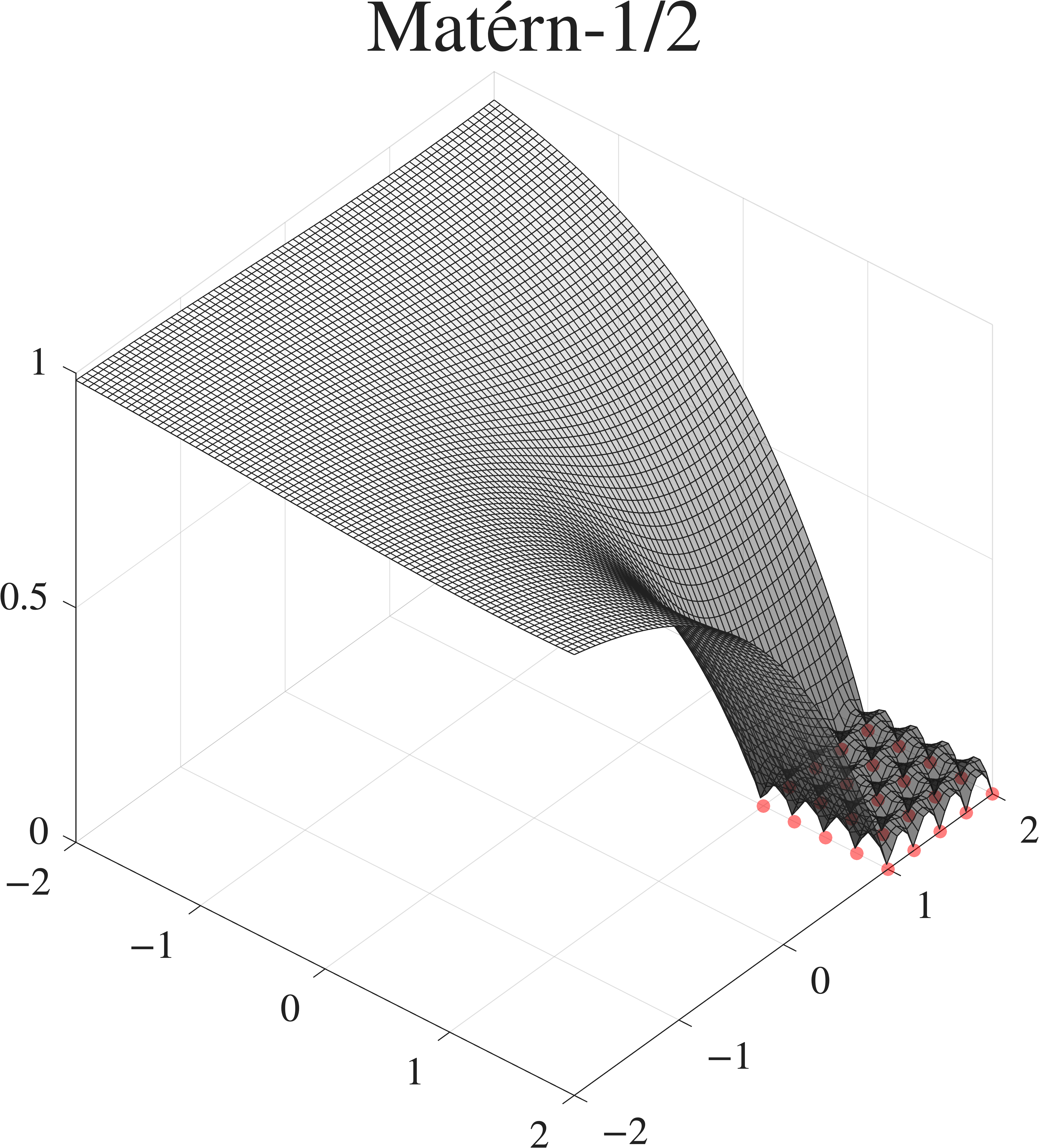}
    \hspace{0.2cm}
    \includegraphics[width=0.3\textwidth]{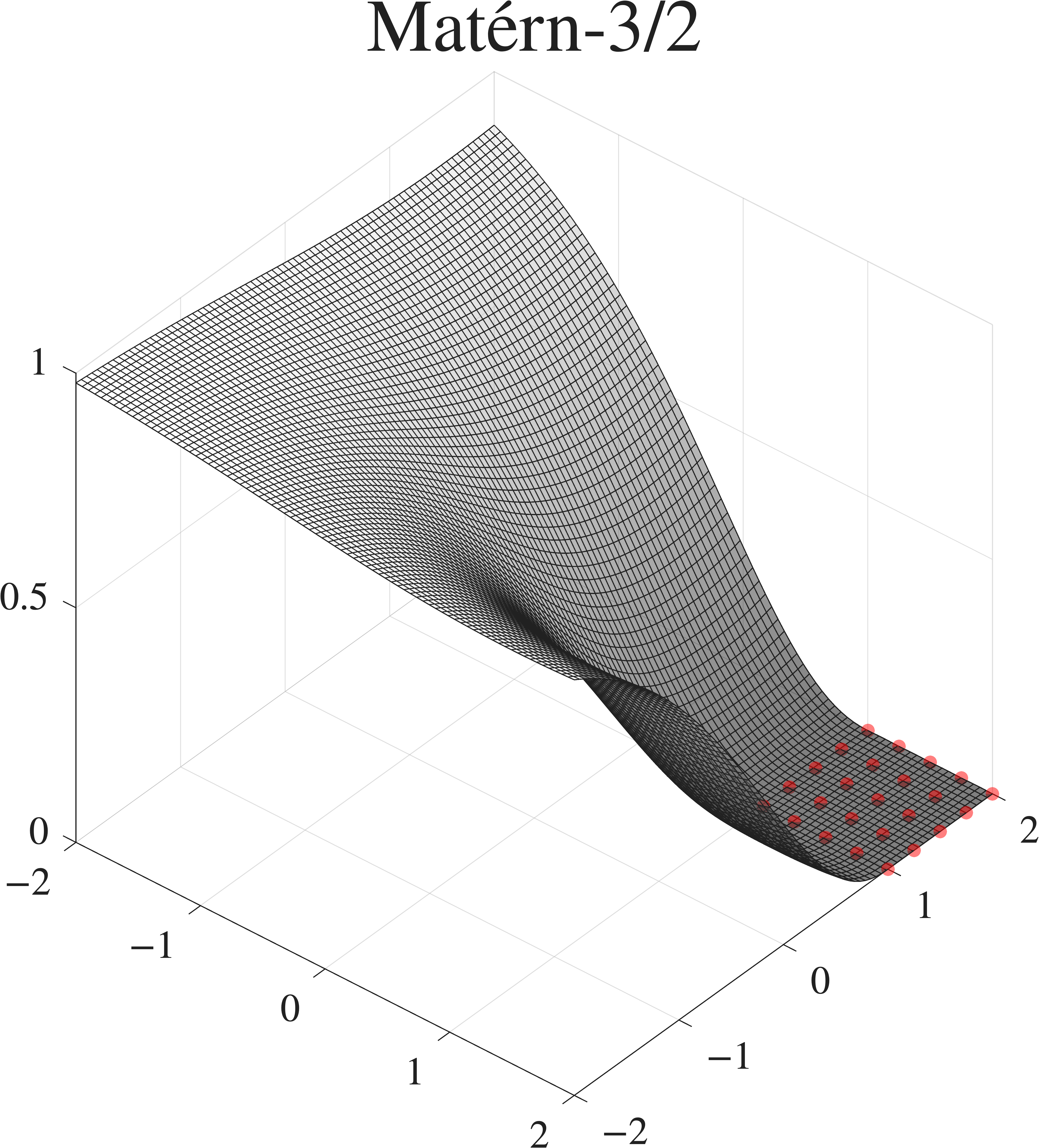}
    \hspace{0.2cm}
    \includegraphics[width=0.3\textwidth]{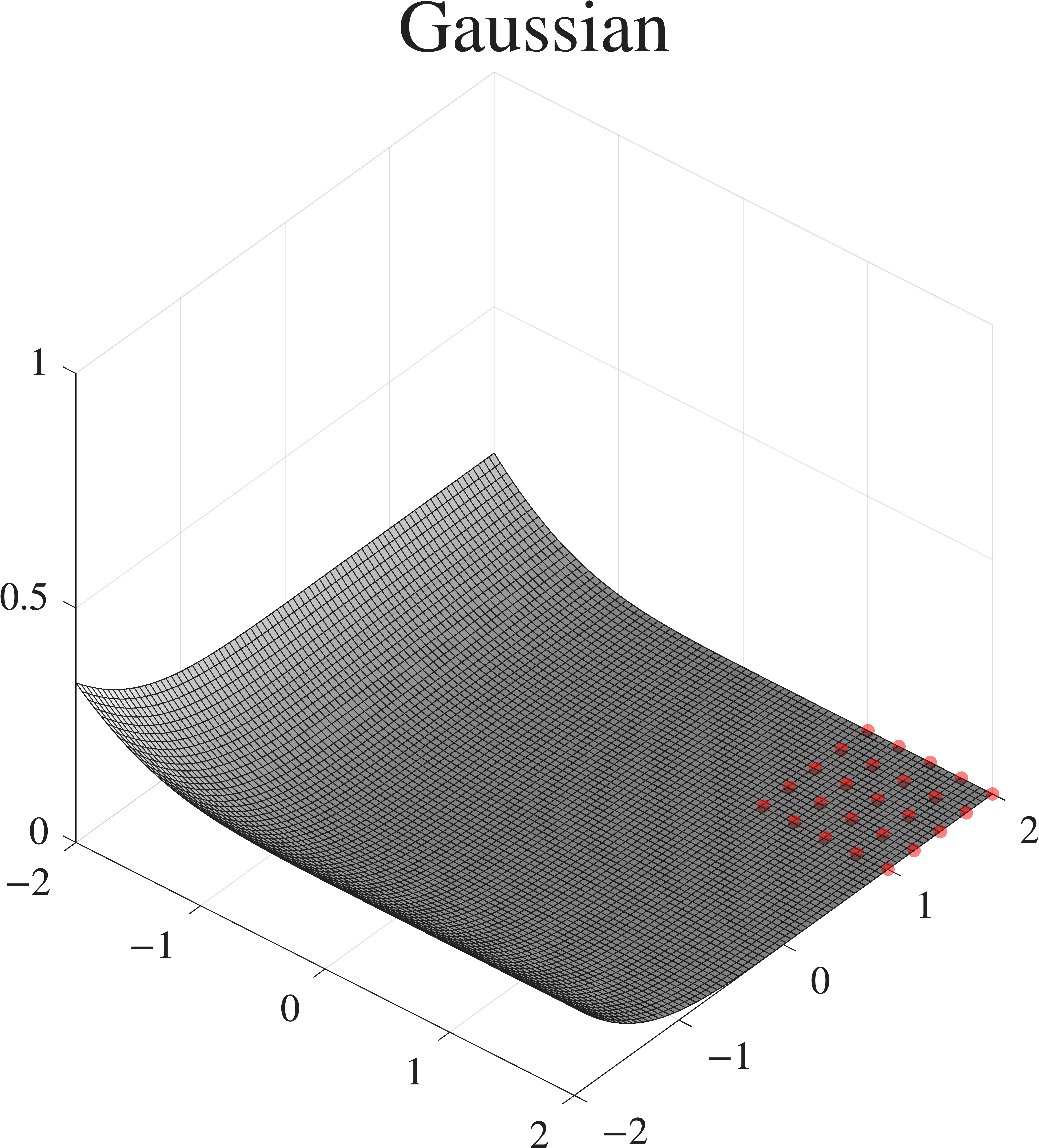}
    \caption{Conditional variance in~\eqref{eq:gp-regression-var} in the interpolatory setting (i.e., $\sigma = 0$) for two Matérn kernels defined in~\eqref{eq:matern} and the Gaussian kernel in~\eqref{eq:gaussian-intro} over the domain $D = [-2, 2]^2 \subset \R^2$.
    All kernels have lengthscale $\ell = 2$. The input points $X_n = \{x_1, \ldots, x_n\}$ (red) form a grid on the sub-domain $[1, 2]^2 \subset [-2, 2]^2$. The Gaussian kernel is covered by Theorem~\ref{thm:interpolation}; how Matérns behave is explained by Theorem~\ref{thm:matern}.}    
    \label{fig:intro-variance}
  \end{figure}    
\end{center}

\paragraph{Problem 2: Numerical ill-conditioning.}
To implement Gaussian process interpolation and many other kernel-based regression and classification methods one must solve a linear system defined by the kernel matrix 
\begin{equation*}
  \b{K}_n = (k(x_i, x_j))_{i,j=1}^n \in \R^{n \times n},
\end{equation*}
where $x_i$ are the input points.
If $k$ is the Gaussian kernel, solving the linear system is practically impossible unless $n$ is very small.
Figure~\ref{fig:intro-condition-number} plots the condition number of $\b{K}_n$ as a function of $n$ for four different kernels.
For the Gaussian kernel the condition number blows up \emph{super-exponentially fast}, which implies that Gaussian process interpolation cannot be implemented in practice. 
To combat numerical ill-conditioning one can then introduce a nugget term, a technique that is equivalent to assuming that the data are noisy.
For example, \texttt{GPy} defaults to the nugget $\sigma^2 = 10^{-8}$ and \texttt{GPyTorch}~\citep{gardner2018gpytorch} to $\sigma^2 = 10^{-6}$.
Opinions differ on whether or not nuggets ought to be used (in particular, see \citealp{GramacyLee2012}).
Regardless, what seems irrefutable to us is that applying a nugget just to make sure that a more or less arbitrarily chosen model can be implemented stably is not the correct approach to take.
That is, if one has little prior information to choose the model with, defaulting to the Gaussian kernel introduces completely unnecessary numerical difficulties that could have been largely avoided by defaulting to a less smooth kernel, such as a Matérn.
It is counterproductive to use the Gaussian kernel if in doing so one must, by introducing a nugget, modify the model.
One could as well start with a numerically more stable kernel.
Ill-conditioning results for the Gaussian kernel are given in Theorem~\ref{thm:ill-conditioning-general} and Corollary~\ref{thm:ill-conditioning-gaussian}.
\\

\begin{figure}[t]
  \begin{center}
  \includegraphics[width=0.8\textwidth]{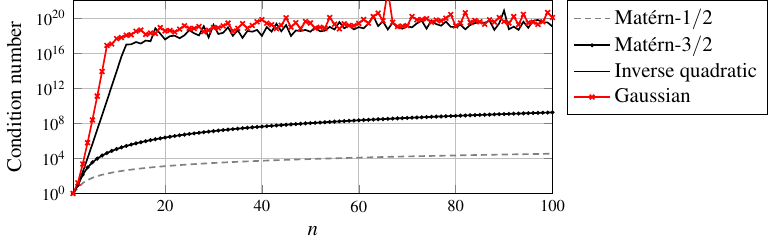}
  \end{center}
  \caption{Condition number [see~\eqref{eq:condition-number}] of the kernel matrix $\b{K}_n = (k(x_i, x_j))_{i,j=1}^n \in \R^{n \times n}$ for Matérn-$1/2$, Matérn-$3/2$ [see~\eqref{eq:matern}], inverse quadratic [see~\eqref{eq:inverse-quadratic}], and Gaussian [see~\eqref{eq:gaussian-intro}] kernels on $\R$. 
  All kernels have lengthscale $\ell = 2$.
  For each $n$ the input points are equispaced on the unit interval. 
  Computations were done in double-precision arithmetic.  
  Results for the Gaussian and inverse quadratic kernels are unreliable after $n \approx 10$ due to limited precision.
  The largest condition number that can be resolved in double-precision arithmetic is somewhere between $10^{16}$ and $10^{20}$. 
  For $n \geq 10$ the actual condition numbers for the Gaussian and inverse quadratic kernels are significantly larger than shown in the figure.
  For example, a \emph{lower bound} on the condition number of $\b{K}_n$ for $n = 21$ and the Gaussian kernel given by Corollary~\ref{thm:ill-conditioning-gaussian} is approximately $1.28 \times 10^{24}$.}
  \label{fig:intro-condition-number}
\end{figure}    

The Gaussian kernel is not the only kernel subject to such problems.
Section~\ref{sec:smoothness-matern} takes a more general approach and identifies classes of stationary kernels that suffer from Problems~1 and~2. 
In short, a stationary kernel whose spectral density decays at least exponentially is analytic and suffers from these problems (Theorem~\ref{thm:analytic}).  
The spectral density of the Gaussian kernel is proportional to $\exp(- c_1 \lVert \omega \rVert_2^2)$ for certain $c_1 > 0$; the inverse quadratic kernel whose behaviour mimics that of the Gaussian in Figure~\ref{fig:intro-condition-number} has spectral density proportional to $\exp(-c_2 \lvert \omega \rvert)$ for $c_2 > 0$ if $d = 1$.
The problems are avoided if the spectral density decays sub-exponentially (Theorems~\ref{thm:matern} and~\ref{thm:gevrey}). 
Matérns, which have little in common with the Gaussian in Figures~\ref{fig:intro-variance} and~\ref{fig:intro-condition-number}, have polynomially decaying spectral densities.
Kernels that we call \emph{Gevrey kernels} have sub-exponential spectral densities of the form $\exp(-\rho \lVert \omega \rVert_2^{\eta})$ for $\rho > 0$ and $\eta \in (0, 1)$.
These kernels are notable in being infinitely differentiable but non-analytic.
Few Gevrey kernels are available in closed form. 
The most important exception is the case $\eta = 2/3$ and $d = 1$, which yields a kernel that can be expressed in terms of a Whittaker function.
To the best of our knowledge, Gevrey kernels have not previously appeared in the literature on Gaussian processes.

In this article we prefer readability over optimality and have strived to provide a rather self-contained account. 
While many results we present could be slightly improved by invoking existing results in the literature, in no case would this alter their interpretation.
For example, it is straightforward to use results in \cite{Yarotsky2013} and \cite{Karvonen2022-power-series} to optimise the exponential term in \Cref{thm:interpolation-1d,thm:interpolation} and Corollary~\ref{thm:ill-conditioning-gaussian}.
\cite{Karvonen2026} gives an account that is similar in flavour to \Cref{sec:smoothness-matern} but based on much more advanced tools from harmonic analysis.

\subsection{Case against the Gaussian in the literature} \label{sec:literature}

We are not the first to take note of Problems~1 and~2.
Although \citet[pp.\@~30, 55 and 69--70]{Stein1999} has laid out a strong case against the use of the Gaussian kernel in modelling due to its infinite smoothness, his arguments need not be persuasive to all practitioners, being either empirical (pp.\@~69--70) or applying to a setting that, as he freely admits, rarely occurs in practice (p.\@~30): ``That is, it is possible to predict $Z(t)$ perfectly for all $t > 0$ based on observing $Z(s)$ for all $s \in (-\epsilon, 0]$ for any $\epsilon > 0$.''
See also \citet[p.\@~406]{HandcockStein1993}.
Such results on ``predicting the future based on the past'' have a long history and go back at least to the works of \citet{Kolmogorov1941} and \citet[p.\@~54]{Wiener1949}.\footnote{See \citet{Poltoratski2012} for a concise description of connections to harmonic analysis.}
A more recent tradition is to be found in kriging and optimisation literature, where the Gaussian kernel has been proved to lack the \emph{no-empty-ball} property~\citep{VazquezBect2010b, VazquezBect2010a, Yarotsky2013, PetitBectVazquez2022}. 
A kernel is said to have this property if the convergence to zero of the conditional variance as the number of data increases is equivalent to the denseness of the sequence of sampling locations in the domain of interest.
In particular, variants of some of our main results could be derived directly from Theorem~2 in \citet{Yarotsky2013}.
The \emph{screening effect} is a related phenomenon discussed in geostatistical literature.
The screening effect is said to hold for a given stochastic process if the prediction at a point $x$ depends ``mostly'' only on observations close to $x$.
\citet{Stein2002, Stein2011, Stein2015} has made this notion precise and shown that the screening effect does not hold if the kernel of a Gaussian process is Gaussian or inverse quadratic.

One attempting to invert the kernel matrix $\b{K}_n = (k(x_i, x_j))_{i,j=1}^n$ when $k$ is Gaussian is bound to observe its extreme ill-conditioning.
This observation has been made repeatedly in the literature.
For example, see Example~1 in \citet{DiamondArmstrong1984}, Section~4.3 in \citet{Lewis1987}, the work of \citet{Posa1989} whose conclusions\footnote{\citet[p.\@~764]{Posa1989}: ``When no physical reason is invoked for the choice of a particular covariance model, a more robust covariance, with respect to some parameters and number of data points, is preferred.''} are much in the spirit of this article, pages~101--2 in \citet{Ababou1994}, Section~2 in \citet{PengWu2014}, page~3063 in \citet{GuWangBerger2018}, and \citet{Belkin2018}.
How condition numbers of kernel matrices behave is relatively well understood in the literature on radial basis functions~\citep{Schaback1995, Diederichs2019}.
See Chapter~12 in \citet{Wendland2005} for a review.
The general picture is that the smoother the kernel, the larger the condition number~\citep[Guideline~3.13]{SchabackWendland2006}.
It is interesting to note that, except when it comes to numerical ill-conditioning, few remarks on the disadvantages of the Gaussian kernel seem to exist in the radial basis function literature, which serves as a foundation for most convergence theory of Gaussian process interpolation.
Convergence results are almost invariably expressed in terms of the fill-distance $h = \sup_{x \in D} \min_{i=1,\ldots,n} \lVert x - x_i \rVert_2$ of the input points $x_1, \ldots, x_n$ in a bounded domain of interest $D \subset \R^d$.
The fill-distance tends to zero if the sequence of input points, $(x_i)_{i=1}^\infty$, is dense in $D$.
Standard references include \citet[Section~11.4]{Wendland2005} and \citet{RiegerZwicknagl2010}.
\citet{Yarotsky2013b} is one of the few exceptions known to us that does not use the fill-distance.
We also remark that \citet{Platte2011} has established interesting connections between approximation by polynomials and Gaussians.

\subsection{What then?} \label{sec:what-then}

To attempt to give a general recommendation of which kernel to use is largely futile.
But let us assume that one has resolved to use a stationary kernel.
Whether or not that is a reasonable thing to do is not a debate that we wish to enter into in this article.\footnote{It probably is not.}
Our main message is that the Gaussian kernel should be avoided. 
Likewise, analytic kernels (i.e., the spectral density decays exponentially or faster) are to be avoided.
We still find the old recommendation by \citet{Stein1999} to use a finitely smooth Matérn a sensible course of action and consider a low-order Matérn, such as one of order $\nu = 3/2$ or $\nu = 5/2$, a reasonable default in the absence of real prior information.
If one is convinced that the latent function is infinitely differentiable but wants to steer clear of the Gaussian kernel (as one should), inverse multiquadrics in~\eqref{eq:multiquadric}, which are ``barely'' analytic, and Gevrey kernels in \Cref{sec:gevrey} are noteworthy options.
As closed form expressions for Gevrey kernels are unavailable when $d > 1$, one would have to resort to product kernels. 
However, this need not be a problem even if one were averse to product kernels, for smooth product kernels tend to look very much like their isotropic cousins (see Figure~\ref{fig:isotropic-product}). 

\begin{center}
  \begin{figure}[t]
    \includegraphics[width=\textwidth]{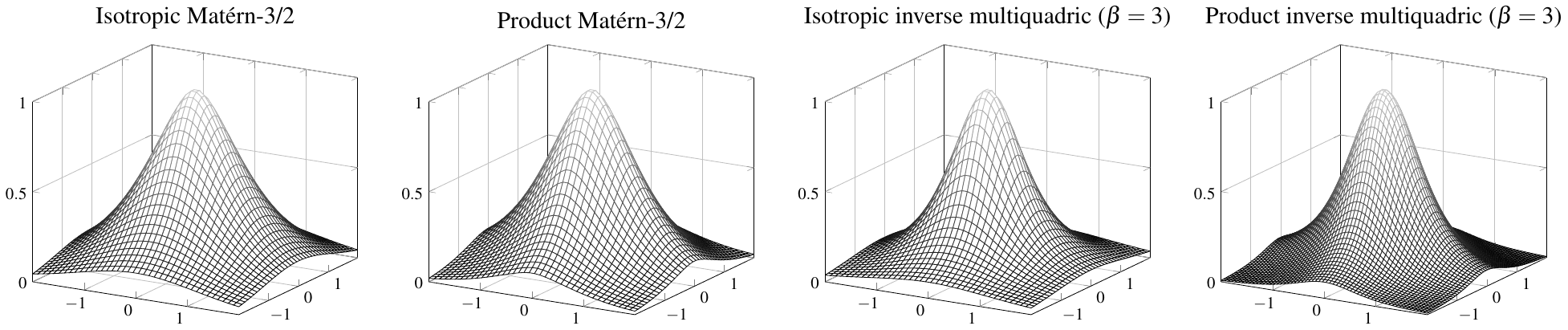}    
    \caption{Isotropic and product forms on $\R^2$ of the Matérn-$3/2$ kernel in~\eqref{eq:matern} and the inverse multiquadric kernel in~\eqref{eq:multiquadric}. 
    The plots depict kernel translates at the origin, $k(0, \cdot)$.
    Isotropic kernels have the form $k(x, y) = \phi(\lVert x - y \rVert_2)$ for $\phi \colon [0, \infty) \to \R$ and product kernels the form $k(x, y) = \phi(\lvert x_1 - y_1 \rvert) \cdots \phi(\lvert x_d - y_d \rvert)$, where $x = (x_1, \ldots, x_d) \in \R^d$ and $y = (y_1, \ldots, y_d) \in \R^d$. All kernels use lengthscale $\ell = 1$.}
    \label{fig:isotropic-product}
  \end{figure}    
\end{center}


\section{Gaussian processes} \label{sec:gps}

\citet{RasmussenWilliams2006} provide a now-standard treatment of Gaussian processes for machine learning.
Other useful references include Chapter~2 in \citet{Berlinet2004}; \citet{Gramacy2020}; and \citet{Kanagawa2018}. 

\subsection{Kernels and priors}

Let $m \colon \R^d \to \R$ be a function and $k \colon \R^d \times \R^d \to \R$ a symmetric and strictly positive-definite kernel, which is to say that
\begin{equation} \label{eq:pd-kernel}
  \sum_{i=1}^n \sum_{j=1}^n a_i a_j k(x_i, x_j) > 0
\end{equation}
for any $n \geq 1$, any pairwise distinct points $x_1, \ldots, x_n \in \R^d$, and any non-zero vector $(a_1, \ldots, a_n) \in \R^n$.
The Gaussian kernel
\begin{equation*}
  k(x, y) = \exp\bigg( \! - \frac{ \norm[0]{x - y}_2^2}{2 \ell^2} \bigg)
\end{equation*}
with lengthscale $\ell > 0$ is a basic example of a strictly positive-definite kernel, which follows from Bochner's theorem~\cite[Sec.\@~6.2]{Wendland2005} and the fact that the Fourier transform of a Gaussian function is Gaussian.
Matérn kernels constitute another important class of strictly positive-definite kernels.
Let $\nu > 0$ be a smoothness parameter and $\mathcal{K}_\nu$ the modified Bessel function of the second kind.
The Matérn kernel of smoothness $\nu$ is given by
\begin{equation} \label{eq:matern}
  k_\nu(x, y) = \frac{2^{1-\nu}}{\Gamma(\nu)} \bigg( \frac{\sqrt{2\nu} \norm[0]{x - y}_2}{\ell} \bigg)^\nu \mathcal{K}_\nu \bigg( \frac{\sqrt{2\nu} \norm[0]{ x - y }_2}{\ell} \bigg) .
\end{equation}
Every Matérn kernel of half-integer order (i.e., $\nu = s + 1/2$ for $s \in \N$) can be written as a product of a polynomial and an exponential.
For example,
\begin{equation*}
  k_{\nu = 1/2}(x, y) = \exp\bigg(\! -\frac{r}{\ell} \bigg) \quad \text{ and } \quad k_{\nu = 3/2}(x, y) = \bigg(1 + \frac{\sqrt{3} r}{\ell} \bigg) \exp\bigg(\! - \frac{\sqrt{3} r}{\ell} \bigg) ,
\end{equation*}
where $r = \lVert x - y \rVert_2$.
Matérns converge pointwise to the Gaussian as $\nu \to \infty$.
We return to this phenomenon in Section~\ref{sec:matern-gaussian-limit}.

A Gaussian process $\fGP \sim \GP(m, k)$ on $\R^d$ with mean $m$ and covariance kernel $k$ is a stochastic process whose finite-dimensional distributions are multivariate Gaussians with mean and covariance determined by $m$ and $k$.
That is, for any $n \geq 1$ and any collection of points $X_n = \{ x_1, \ldots, x_n \} \subset \R^d$,
\begin{equation*}
  \begin{bmatrix}
    \fGP(x_1) \\ \vdots \\ \fGP(x_n)
  \end{bmatrix}
  \sim
  \mathrm{N}
  \left(
  \begin{bmatrix}
    m(x_1) \\ \vdots \\ m(x_n)
  \end{bmatrix},
  \begin{bmatrix}
    k(x_1, x_1) & \cdots & k(x_1, x_n) \\
    \vdots & \ddots & \vdots \\
    k(x_n, x_1) & \cdots & k(x_n, x_n)
  \end{bmatrix}
  \right)
  \eqqcolon
  \mathrm{N}( \b{m}_n, \b{K}_n ).
\end{equation*}
The covariance matrix $\b{K}_n$, henceforth called \emph{kernel matrix}, is strictly positive-definite and thus invertible whenever the points in $X_n$ are pairwise distinct.
In particular,
\begin{equation*}
  \mathbb{E}[ \fGP(x) ] = m(x) \quad \text{ and } \quad \mathrm{Cov}[ \fGP(x), \fGP(y) ] = k(x, y)
\end{equation*}
for all $x, y \in \R^d$.
The covariance kernel determines what the samples from the Gaussian process look like, or how they fluctuate.
In particular, the smoothness (i.e., the degree of differentiability) of the kernel is inherited by the samples.
This is illustrated in Figure~\ref{fig:kernels-and-priors}.

\begin{center}
  \begin{figure}[t]
    \includegraphics[width=0.3\textwidth]{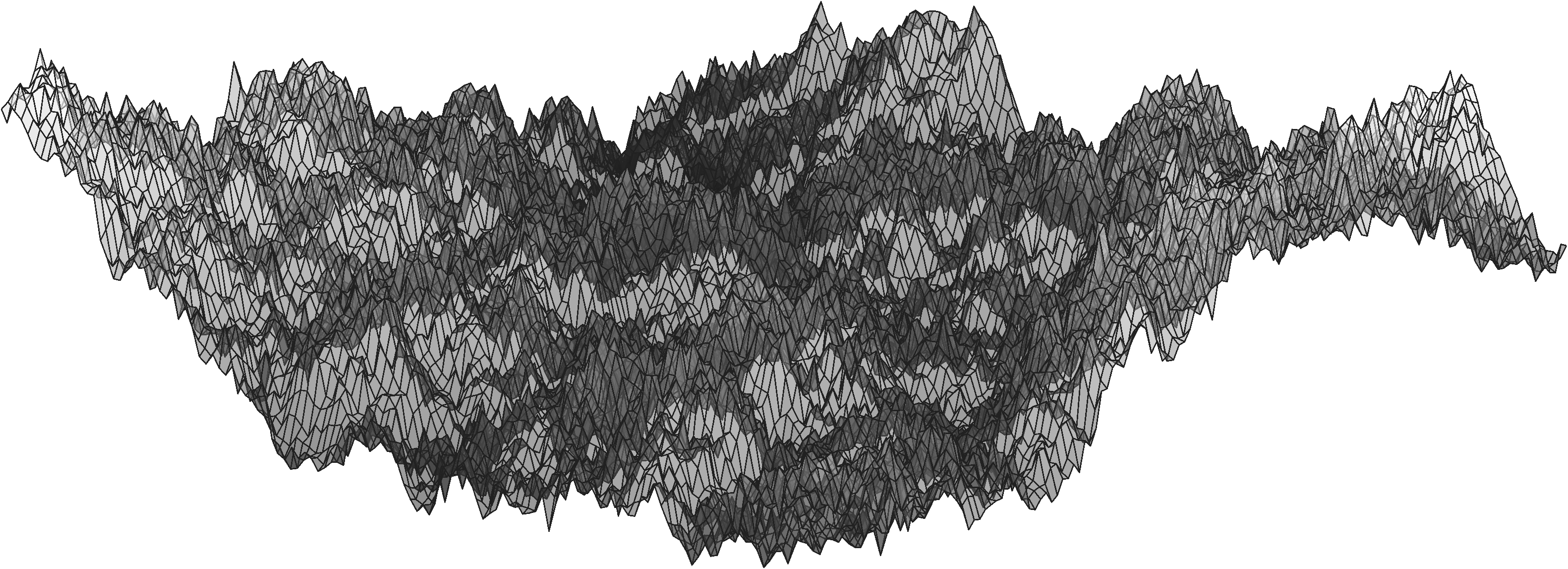}
    \hspace{0.2cm}
    \includegraphics[width=0.3\textwidth]{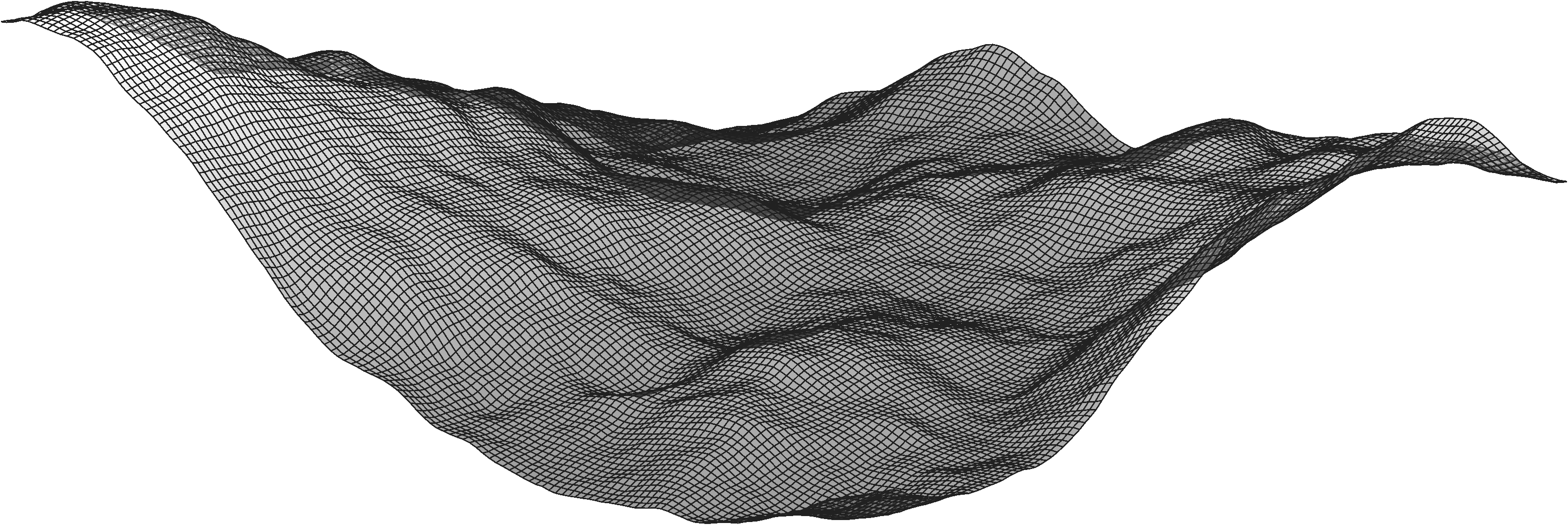}
    \hspace{0.2cm}
    \includegraphics[width=0.3\textwidth]{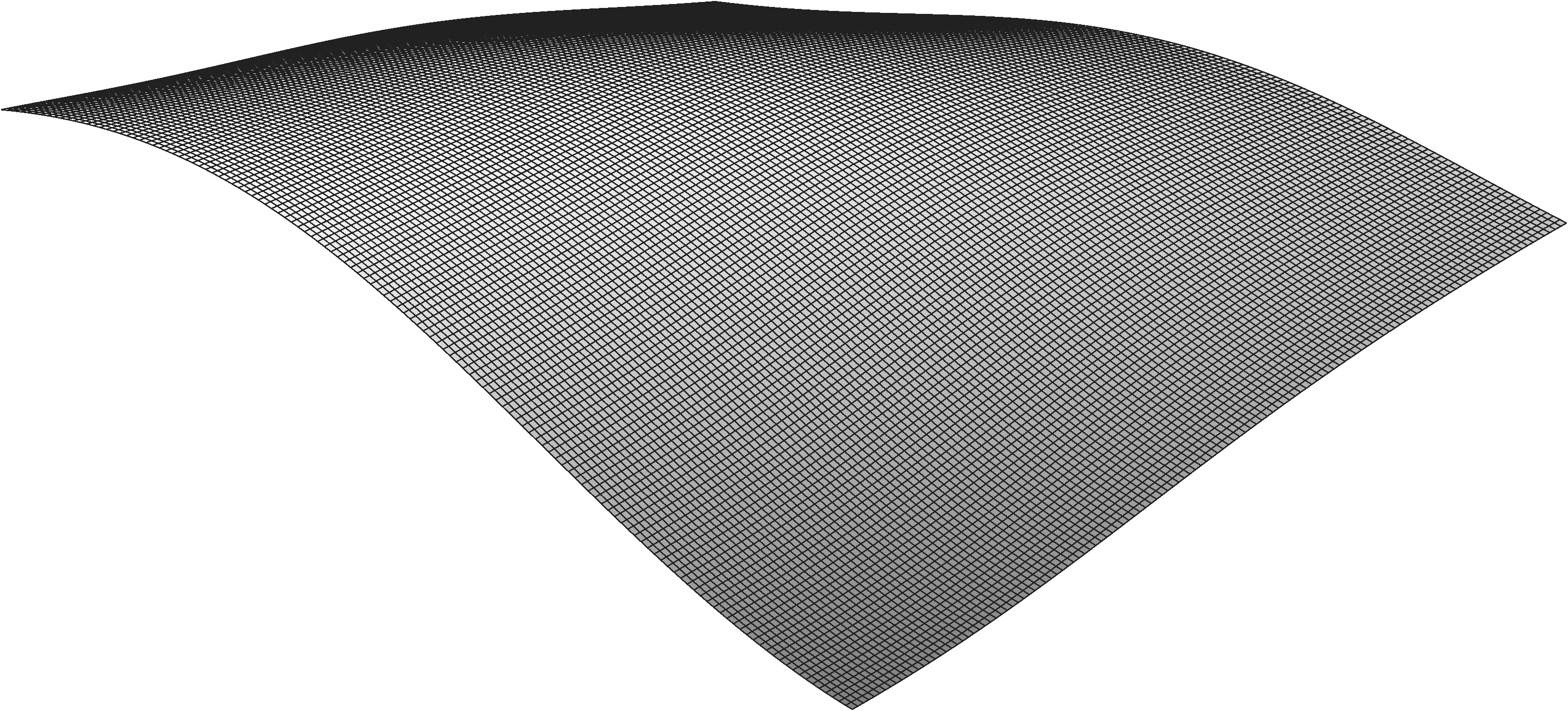}
    \caption{Samples from Gaussian processes on $\R^2$ with three different kernels. \emph{Left:} Matérn kernel with $\nu = 1/2$ (continuous but non-differentiable samples). \emph{Middle:} Matérn kernel with $\nu = 3/2$ (once differentiable samples). \emph{Right:} Gaussian kernel (infinitely differentiable samples).}
    \label{fig:kernels-and-priors}
  \end{figure}    
\end{center}

\subsection{Gaussian process regression and interpolation} \label{sec:gp-regression}

Let $D$ be a non-empty subset of $\R^d$ and $f_0 \colon D \to \R$ an unknown latent function.
In Gaussian process regression, the unknown function $f_0$ is modelled as a Gaussian process $\fGP \sim \GP(m, k)$ that is then conditioned on a set of training data $\mathcal{D}_n = \{(x_i, y_i)\}_{i=1}^n$ which consists of potentially noisy observations $y_i$ of $f_0$ at a collection of $n$ input points, $X_n = \{x_1, \ldots, x_n\} \subset D$.
We always tacitly assume that the input points are pairwise distinct.
It is typical to assume that the observations are corrupted by Gaussian noise with standard deviation $\sigma \geq 0$:
\begin{equation*}
  y_i = f_0(x_i) + \varepsilon_i, \quad \text{ where } \quad \varepsilon_i \sim \mathrm{N}(0, \sigma^2) \text{ are i.i.d.}
\end{equation*}
The conditional process, $\fGP \mid \mathcal{D}_n$, is a Gaussian process.
Its mean and variance are
\begin{equation*}
  \mu_\sigma( x \mid \mathcal{D}_n ) = \mathbb{E}[ \fGP(x) \mid \mathcal{D}_n ] = m(x) + \b{k}_n(x)^\T (\b{K}_n + \sigma^2 \b{I}_n)^{-1} ( \b{y}_n - \b{m}_n )
\end{equation*}
and
\begin{equation} \label{eq:gp-regression-var}
  \V_\sigma( x \mid \mathcal{D}_n ) = \mathrm{Var}[ \fGP(x) \mid \mathcal{D}_n ] = k(x, x) - \b{k}_n(x)^\T (\b{K}_n + \sigma^2 \b{I}_n)^{-1} \b{k}_n(x),
\end{equation}
where $\b{k}_n(x) = (k(x, x_1), \ldots, k(x, x_n))$ and $\b{y}_n = (y_1, \ldots, y_n)$ are column vectors and $\b{I}_n$ the $n \times n$ identity matrix.
If $\sigma = 0$, the latent function is observed without noise, so that $y_i = f_0(x_i)$.
In this case we use the more concise notation
\begin{equation} \label{eq:gp-interpolation}
    \mu( x \mid \mathcal{D}_n ) = m(x) + \b{k}_n(x)^\T \b{K}_n^{-1} ( \b{y}_n - \b{m}_n ) \:\: \text{ and } \:\: \V( x \mid \mathcal{D}_n ) = k(x, x) - \b{k}_n(x)^\T \b{K}_n^{-1} \b{k}_n(x).
\end{equation}    
To distinguish the cases $\sigma > 0$ and $\sigma = 0$, we speak of Gaussian process \emph{interpolation} when $\sigma = 0$.
Most of our arguments against the Gaussian kernel apply to interpolation.
The conditional mean and variance for both $\sigma = 0$ and $\sigma > 0$ are depicted in Figure~\ref{fig:gp-regression}.
Observe that for $\sigma = 0$ the conditional mean interpolates the data, in that $\mu(x_i \mid \mathcal{D}_n) = f_0(x_i)$ and $\V(x_i \mid \mathcal{D}_n) = 0$ for each $i \in \{ 1, \ldots, n \}$.
The conditional mean provides point estimates for the unobserved outputs $f_0(x)$ while the conditional variance quantifies uncertainty in these estimates via, for example, credible intervals as in Figure~\ref{fig:gp-regression}.
Note that the conditional variance in~\eqref{eq:gp-regression-var} and~\eqref{eq:gp-interpolation} does not depend on the observations $y_i$. 
The gist of our case against the Gaussian kernel is that it gives rise to a conditional variance that is far too small, a phenomenon that is to some extent observable in Figure~\ref{fig:gp-regression}.

\begin{center}
  \begin{figure}[t]
    \includegraphics[width=\textwidth]{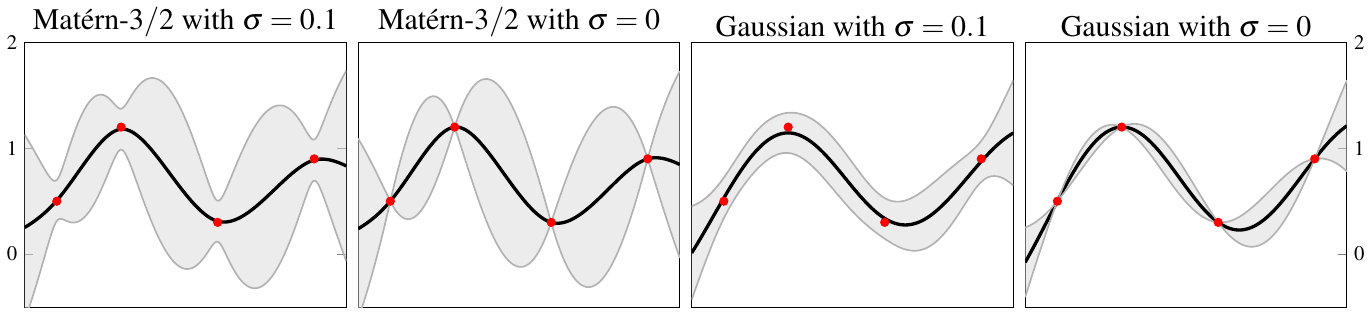}
    \caption{The conditional mean (black) and the pointwise 95\% credible intervals (shaded) obtained from the conditional variance over the interval $[0, 1]$ for the Matérn-$3/2$ and Gaussian kernels with $\ell = 0.3$. The data consist of four observations that are indicated in red. Observe that the mean interpolates the data when $\sigma = 0$.}
    \label{fig:gp-regression}
  \end{figure}    
\end{center}

\section{Why not to use the Gaussian kernel} \label{sec:main-sec}

The gist of our case against the Gaussian kernel 
\begin{equation} \label{eq:gaussian}
  k(x, y) = \exp\bigg( \! - \frac{ \norm[0]{x - y}_2^2}{2 \ell^2} \bigg),
\end{equation}
where $x, y \in \R^d$, is that it is by far the \emph{brittlest} among all commonly used kernels for Gaussian process regression in machine learning and spatial statistics.
The brittleness is manifested via (a) overconfident uncertainty quantification and (b) numerical ill-conditioning.
These issues are discussed in \Cref{sec:uq} and \Cref{sec:ill-conditioning}, respectively.

\subsection{Overconfident uncertainty quantification} \label{sec:uq}

A Gaussian process model is \emph{overconfident} if the conditional standard deviation, $\V_\sigma(\cdot \mid \mathcal{D}_n)^{1/2}$, is ``much'' smaller than the true prediction error, $\lvert f_0 - \mu_\sigma(\cdot \mid \mathcal{D}_n) \rvert$, that it purports to quantify.
We refer to \citet{Karvonen2020-MLE} for a more thorough discussion on overconfidence.
Typically it is unreasonable to expect that a latent function can be predicted well over the entire domain~$D$ if observations are available only on a small sub-domain.
The Gaussian kernel does \emph{not} behave as this basic intuition suggests.
Recall from \Cref{sec:gp-regression} that $\sigma \geq 0$ is the standard deviation of the observation noise.

\begin{theorem}[Overconfidence for $d=1$] \label{thm:interpolation-1d}
  Let $d = 1$ and $\sigma = 0$.
  Consider the Gaussian kernel in~\eqref{eq:gaussian} and $D = [a, b]$ for $a < b$.
  If $X_n = \{ x_1, \ldots, x_n \} \subset D$ are any points, then
  \begin{equation*} 
    \sup_{x \in D} \V( x \mid \mathcal{D}_n ) \leq 2 \bigg( \frac{\sqrt{2}(b-a)}{\ell} \bigg)^{2n} \frac{1}{n!} .
  \end{equation*}
\end{theorem}
\begin{proof}
  See \Cref{sec:proofs-main-sec}.
  The proof is based on the classical estimate for the error of polynomial interpolation and the connection between the conditional variance and the pointwise worst-case error in the reproducing kernel Hilbert space of $k$.
\end{proof}

\begin{theorem}[Overconfidence for $d \geq 1$] \label{thm:interpolation}
  Let $d \geq 1$ and $\sigma = 0$.
  Consider the Gaussian kernel in~\eqref{eq:gaussian} and a bounded $D \subset \R^d$.
  If the set $X_n = \{x_1, \ldots, x_n\} \subset D$ contains a Cartesian product $X_1' \times \cdots \times X_d'$ of $d$ sets of $m$ pairwise distinct reals, then
  \begin{equation*}
    \sup_{x \in D} \V( x \mid \mathcal{D}_n ) \leq 2d \bigg( \frac{\sqrt{2}(b-a)}{\ell} \bigg)^{2m} \frac{1}{m!} ,
  \end{equation*}
  where $a < b$ are such that $D \subset [a, b]^d$.
  In particular, if $X_n = X_1' \times \cdots \times X_d'$, then
  \begin{equation*}
    \sup_{x \in D} \V( x \mid \mathcal{D}_n ) \leq  2d \bigg( \frac{\sqrt{2}(b-a)}{\ell} \bigg)^{2n^{1/d}} \frac{1}{(n^{1/d})!} .
  \end{equation*}  
\end{theorem}
\begin{proof}
  See \Cref{sec:proofs-main-sec}.
  The proof consists of a standard tensor product argument combined with \Cref{thm:interpolation-1d}.
\end{proof}

\begin{center}
  \begin{figure}[t]
    \includegraphics[width=\textwidth]{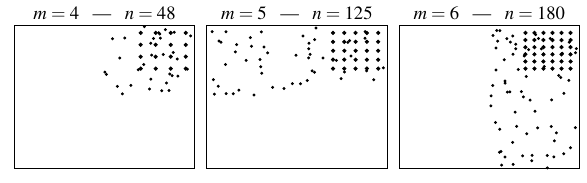}
    \caption{Three input point sets $X_n$ on $D = [-1, 1]^2 \subset \R^2$ that Theorem~\ref{thm:interpolation} applies to. The Cartesian grids $X_1' \times X_2' \subset X_n$ contained in the sets are emphasised.
    A two-dimensional Cartesian grid consists of $m \times m$ points.}    
    \label{fig:grids}
  \end{figure}
\end{center}

\Cref{thm:interpolation-1d,thm:interpolation} state that the conditional variance tends to zero with a factorial rate on the entire domain $D$ essentially regardless of where the input points are placed.
These theorems are variants of the results discussed in \Cref{sec:literature} on predicting the future of a process from past only~\citep{Kolmogorov1941} and the no-empty-ball property~\citep{VazquezBect2010b}. 
The conclusion is stronger if $d = 1$, as then the conditional variance tends to zero for \emph{any} sequence of sampling points.
For $d > 1$ the ``bad'' point configurations are not artificial as it suffices that $X_n$ contains a $d$-fold Cartesian product (see Figure~\ref{fig:grids}).\footnote{We believe that the estimates of \Cref{thm:interpolation} remain true in some form for ``almost any'' $X_n$. However, it does not appear easy to prove this.}

Overconfidence of uncertainty quantification with the Gaussian kernel was illustrated already in \Cref{fig:intro-variance}, which showed that the Gaussian kernel produces a much smaller variance than Matérns. 
\Cref{fig:variance} shows how the conditional variance tends to zero everywhere on the domain $D = [-2, 2]^2$ as more data are obtained even though all sampling points are contained in the small sub-domain $D_0 = [1, 2]^2$.
In a Gaussian process model based on the Gaussian kernel \emph{local information is therefore global information}: the model thinks that it can learn the latent function globally from local information alone and expresses this belief by shrinking the conditional variance everywhere.
This behaviour is akin to (and, as we shall discuss in \Cref{sec:analytic}, a direct consequence of) the fact that an analytic function $f \colon \R^d \to \R$ with infinite radius of convergence can be recovered at any $x \in \R^d$ from the values of its derivatives at a single point $x_0 \in \R^d$ via its Taylor expansion
\begin{equation*}
  f(x) = \sum_{\alpha \in \N_0^d} \frac{\mathrm{D}^\alpha f (x_0)}{\alpha!} (x - x_0)^\alpha .
\end{equation*}
None of this needs to be a problem if the user truly believes that their latent function can be learned from local information.
But more often than not the Gaussian kernel is used simply as a convenient default.
In such a situation one runs the serious risk of unintentional overconfident uncertainty quantification if, for one reason or another, sampling leaves large holes in which the conditional variance is nevertheless very small.
Even when the sampling points are space-filling, the conditional variance is likely to decay much faster than is desirable.
While Matérns and other finitely smooth kernels can induce overconfidence, it is much less severe than under the Gaussian kernel. 
Moreover, data-driven estimation of hyperparameters for Matérns and related kernels often suffices to fix the problem~\citep{Szabo2015, Karvonen2020-MLE, Naslidnyk2025, KarvonenBachoc2025}.

\begin{remark} \label{remark:parameter-estimation}
Hyperparameter estimation cannot fix the global decay of the variance, the fundamental problem associated to the Gaussian kernel that was demonstrated by \Cref{thm:interpolation-1d,thm:interpolation}.
The problem cannot be circumvented by estimating $\ell$ (or a multiplicative scaling parameter in front of the kernel) from the data.
For selecting a lengthscale $\ell_n$, data-driven or not, such that the variance decays with a polynomial rate $n^{-2s}$ for some $s > 0$, or any other rate more modest than factorial, in the sub-domain that contains the inputs (e.g., $[1, 2]^2$ in \Cref{fig:variance}) simply forces the variance to decay with the same rate everywhere on the domain.
For example, setting $\ell_n = \sqrt{2e}(b-a) n^{-1/2 + s/n - 1/(4n)}$ in \Cref{thm:interpolation-1d} and using Stirling's approximation for the factorial yields
\begin{equation*}
  \sup_{x \in D} \V(x \mid \mathcal{D}_n) = O(n^{-2s}).
\end{equation*}
This is the contraction rate associated to a Matérn kernel of order $\nu = s$.
However, for Matérns this rate is obtained only when the input sequence is dense in $D$ (see \Cref{thm:matern}); for the Gaussian kernel with estimated $\ell$ there is no such requirement. 
One could, of course, use a spatially varying lengthscale (or scaling parameter), but this brings about a host of complications common to such non-stationary models. 
We return to hyperparameter estimation and more general kernel learning in \Cref{sec:kernel-learning}.
\end{remark}

\begin{center}
  \begin{figure}[t]
    \includegraphics[height=0.3\textwidth]{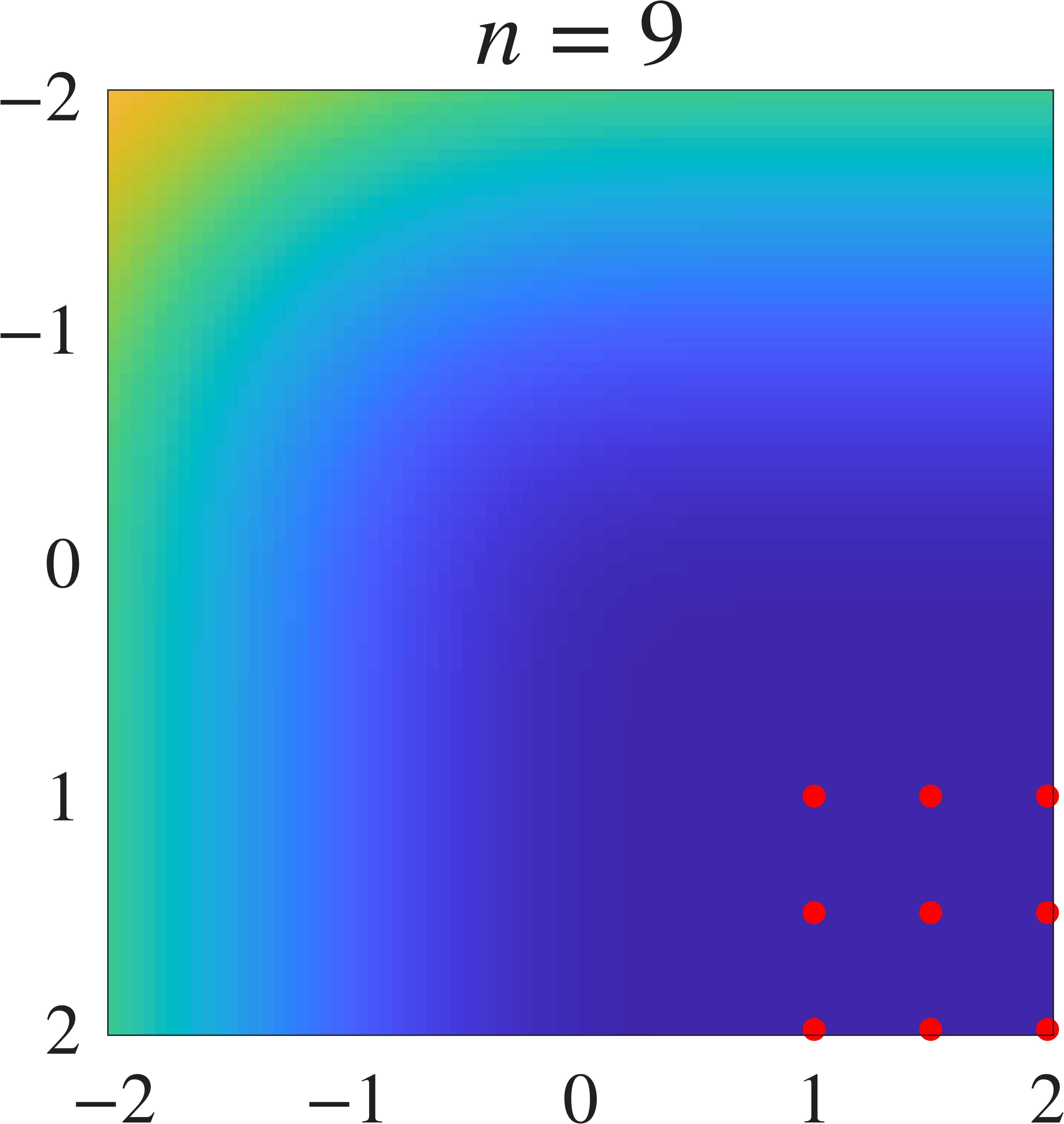}
    \hspace{0.2cm}
    \includegraphics[height=0.3\textwidth]{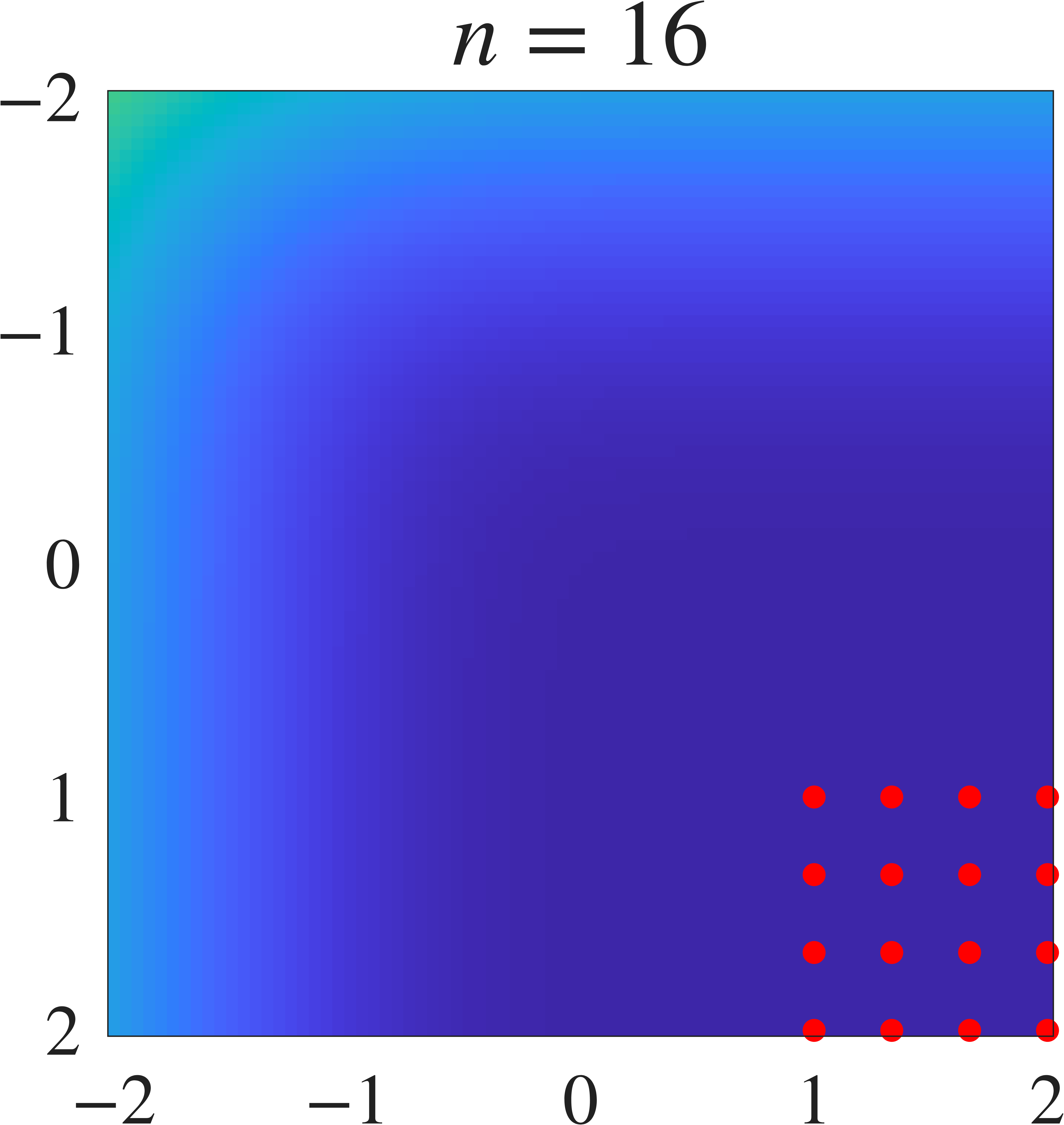}
    \hspace{0.2cm}
    \includegraphics[height=0.3\textwidth]{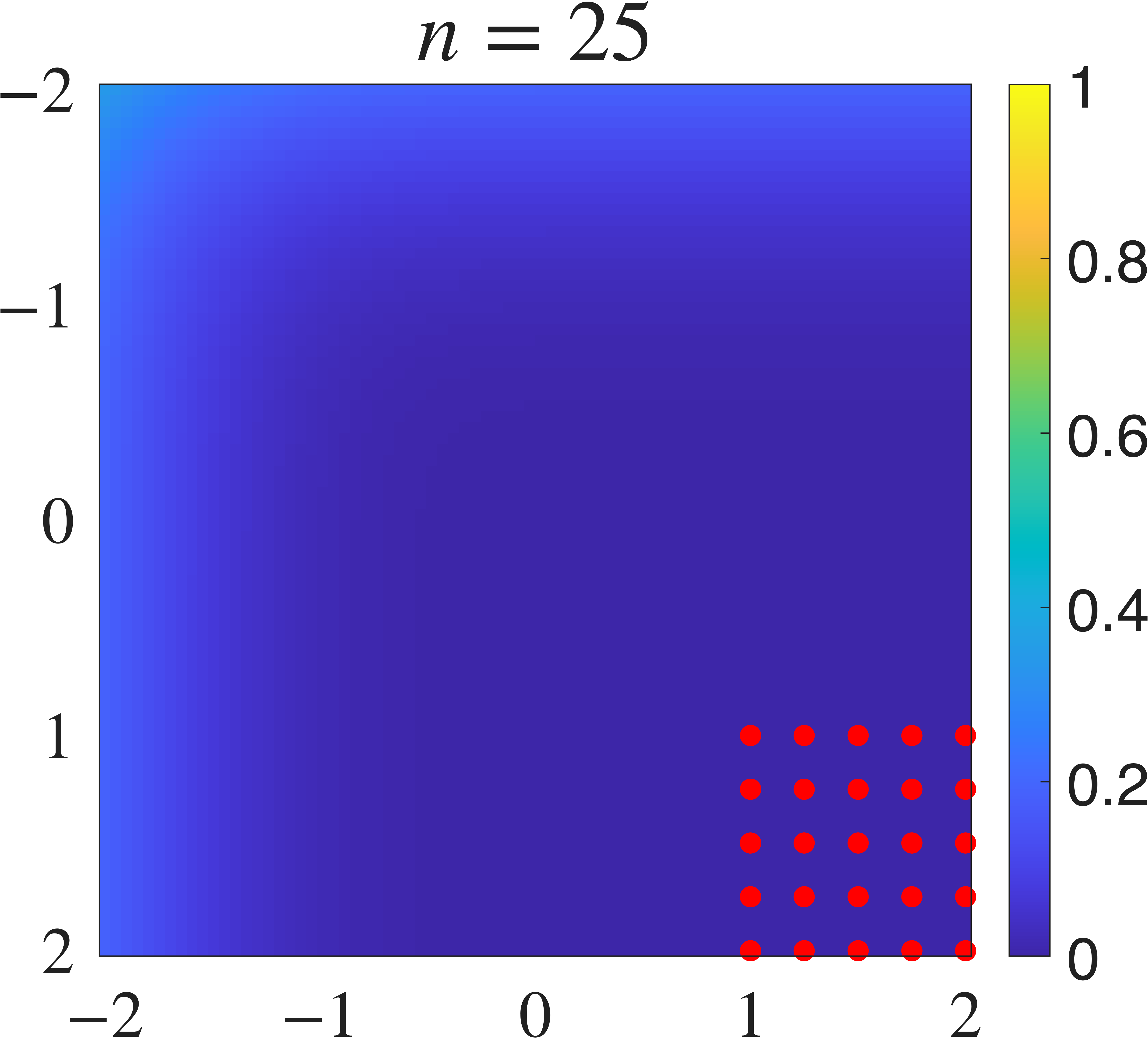}
    \caption{The conditional variance over $[-2, 2]^2 \subset \R^2$ for the Gaussian kernel in~\eqref{eq:gaussian} with $\ell = 2$ when the input points are located in the sub-domain $[1, 2]^2$. As the number of points increases, the variance decreases rapidly over the entire domain even though the points cover only the small sub-domain.}
    \label{fig:variance}
  \end{figure}    
\end{center}

\Cref{thm:interpolation} requires that the set of sampling points contain a tensor grid.
The proof is based on standard tensor product arguments for transferring one-dimensional bounds to higher dimensions.
The conclusion that the variance tends to zero remains true whenever the sequence of points is dense in some open subset of $D$ (see \Cref{thm:regression}).
However, in that case we cannot deduce any rate for $d > 1$.
Moreover, the following proposition shows that there are point configurations for which the variance remains bounded from below when $d > 1$.

\begin{proposition} \label{prop:gaussian-lower-bound}
  Let $d \geq 1$ and $\sigma \geq 0$.
  Consider the Gaussian kernel in~\eqref{eq:gaussian}.
  Let $R_0 \subset \R^d$ be the set of roots of a non-trivial $d$-variate polynomial.
  If $X_n = \{x_1, \ldots, x_n\}$ for any sequence $\{ x_i \}_{i=1}^\infty \subset R_0$ of pairwise distinct points, then for every $x \notin R_0$ there is $c_x > 0$ such that
  \begin{equation*}
    \V_\sigma(x \mid \mathcal{D}_n) \geq c_x
  \end{equation*}
  for every $n \in \N$.
\end{proposition}
\begin{proof}
  See \Cref{sec:proofs-main-sec}.
\end{proof}

When $d = 1$, the set $R_0$ of roots of a non-trivial polynomial is finite. 
However, when $d > 1$, the set of roots is in most cases infinite.
Easy examples of infinite $R_0$ are hyperplanes and the unit sphere $\Set{x \in \R^d}{ \lVert x \rVert_2 = 1}$, which is the set of roots of the polynomial $1 - (x_1^2 + \cdots + x_d^2)$ of total degree 2.
More examples are given in \Cref{fig:polynomial-roots}.

\begin{center}
  \begin{figure}[t]
    \includegraphics[width=\textwidth]{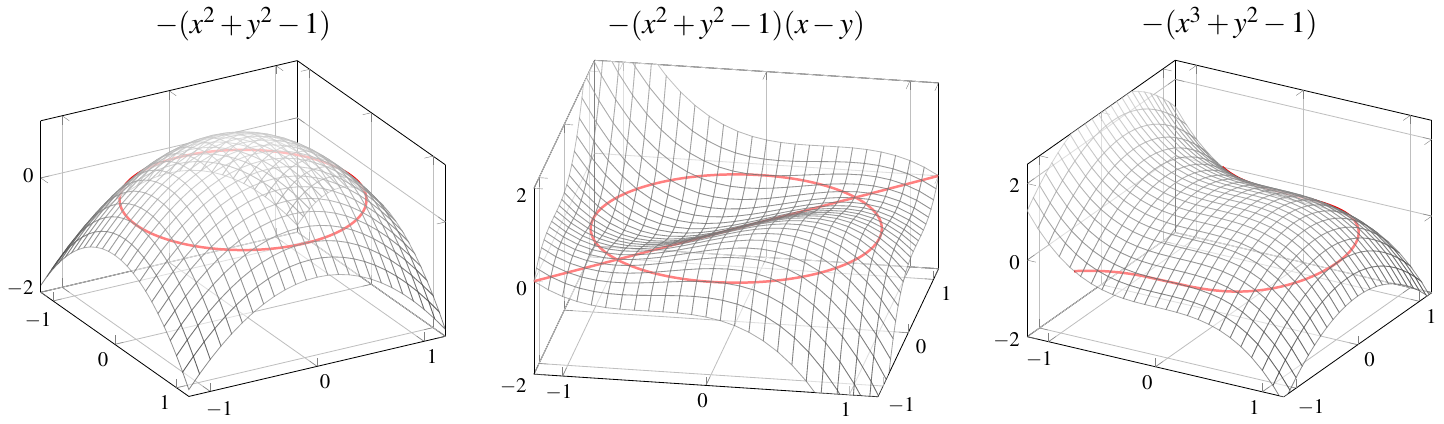}
    \caption{The three bivariate polynomials $-(x^2 + y^2 - 1)$, $-(x^2 + y^2 - 1)(x - y)$, and $-(x^3 + y^2 - 1)$ along with their roots (in red). Observe that in each case the set of roots is infinite.}
    \label{fig:polynomial-roots}
  \end{figure}    
\end{center}

\Cref{thm:interpolation-1d,thm:interpolation} show that the variance tends to zero very fast in the interpolatory setting (i.e., $\sigma = 0$).
The following theorem applies to both interpolation and regression ($\sigma > 0$).

\begin{theorem}[Regression for $d \geq 1$] \label{thm:regression}
  Consider the Gaussian kernel in~\eqref{eq:gaussian} for $d \geq 1$ and $\sigma \geq 0$.  
  If $X_n = \{x_1, \ldots, x_n\}$ for any sequence $\{ x_i \}_{i=1}^\infty \subset \R^d$ of pairwise distinct points whose closure has non-empty interior, then
  \begin{equation*}
    \sup_{x \in K} \V_\sigma( x \mid \mathcal{D}_n ) \to 0 \quad \text{ as } \quad n \to \infty \quad \text{ for any compact } \quad K \subset \R^d .
  \end{equation*}
\end{theorem}
\begin{proof}
  This is a special case of \Cref{thm:analytic}.
\end{proof}

There are some noteworthy differences between \Cref{thm:interpolation-1d,thm:interpolation} and \Cref{thm:regression}.
\Cref{thm:regression} provides no \emph{rate} of convergence even in the interpolatory case ($\sigma = 0$) in which the former theorems imply super-exponential rates. 
However, unlike \Cref{thm:interpolation}, which requires that the set of input points contain a grid, \Cref{thm:regression} makes minimal assumptions about the inputs, only requiring that they be dense in some set.

\subsection{Numerical ill-conditioning} \label{sec:ill-conditioning}

Our second point is that the Gaussian kernel is numerically brittle.
The computation of posterior distributions in Gaussian process regression inevitably involves the kernel covariance matrix $\b{K}_n = (k(x_i, x_j))_{i,j=1}^n$ (see \Cref{sec:gp-regression}).
As we saw in \Cref{sec:literature}, it has been repeatedly observed that to use the Gaussian kernel can be challenging due to exceptionally poor conditioning of the covariance matrix.
If it is assumed that the observations are corrupted by zero-mean Gaussian noise with variance $\sigma^2$, the matrix that one needs to invert to compute the conditional mean and variance is $\b{K}_n + \sigma^2 \b{I}_n$, where $\b{I}_n$ is the identity matrix.
In this case numerical stability is a non-issue unless $\sigma$ is very small because the condition number of $\b{K}_n + \sigma^2 \b{I}_n$ is bounded from above by $(n + \sigma^2)/\sigma^2$.

However, in tasks such as emulation of computer experiments~\citep{Sacks1989} or probabilistic numerical computation~\citep{Hennig2022} it is natural to assume that the data are noiseless (i.e., $\sigma = 0$).
Recall that the condition number of a positive-definite matrix $\b{A}$ is
\begin{equation} \label{eq:condition-number}
  \mathrm{cond}(\b{A}) = \frac{\lambda_{\max}(\b{A})}{\lambda_{\min}(\b{A})},
\end{equation}
where $\lambda_{\max}(\b{A})$ and $\lambda_{\min}(\b{A})$ are the largest and smallest eigenvalues of $\b{A}$.
The magnitude of the condition number of a covariance matrix is connected to the magnitude of the conditional variance.

\begin{theorem} \label{thm:ill-conditioning-general}
  Consider a strictly positive-definite stationary kernel $k(x, y) = \Phi(x - y)$. If $\sigma = 0$ and $X_{n+1} = \{x_1, \ldots, x_{n+1}\} \subset \R^d$, then
  \begin{equation*}
    \mathrm{cond}( \b{K}_{n+1} ) \geq \frac{\Phi(0)}{\V(x_{n+1} \mid \mathcal{D}_{n})} .
  \end{equation*}
\end{theorem}
\begin{proof}
  Estimates in Section~12.1 of \citet{Wendland2005} give $\lambda_{\min}(\b{K}_{n+1}) \leq \V(x_{n+1} \mid \mathcal{D}_{n})$.
  Because the trace is the sum of eigenvalues and $\mathrm{tr}(\b{K}_{n+1}) = (n+1) \Phi(0)$ by stationarity, not all eigenvalues of $\b{K}_{n+1}$ can be less than $\Phi(0)$.
\end{proof}

By plugging in the estimates from \Cref{thm:interpolation-1d,thm:interpolation} we obtain the following result on the condition number of the covariance matrix of the Gaussian kernel:

\begin{corollary} \label{thm:ill-conditioning-gaussian}
  Consider the Gaussian kernel in~\eqref{eq:gaussian} and $X_{n+1} = X_n \cup \{x_{n+1}\}$.
  Then
  \begin{equation*}
    \mathrm{cond}( \b{K}_{n+1} ) \geq \frac{1}{2} \bigg( \frac{\sqrt{2}(b-a)}{ \ell} \bigg)^{-2n} n! \quad \text{ and } \quad \mathrm{cond}( \b{K}_{n+1} ) \geq \frac{1}{2d} \bigg( \frac{\sqrt{2}(b-a)}{ \ell} \bigg)^{-2m} m!
  \end{equation*}
  in the settings of \Cref{thm:interpolation-1d} ($d = 1$) and \Cref{thm:interpolation} ($d \geq 1$), respectively.
\end{corollary}

The condition number thus grows with a factorial rate.
This means that the limited numerical precision prevents the implementation of Gaussian process interpolation if the Gaussian kernel is used.
In the dominant double-precision floating-point arithmetic a condition number of about $10^{16}$ is the most that can be handled without a loss of significant precision.
In \Cref{fig:intro-condition-number} we saw that surprisingly few points are needed to lose numerical precision.
The condition number $\mathrm{cond} (\b{K}_n) \approx 10^{16}$ was exceeded with only $n = 8$ points.
The same figure showed that finitely smooth Matérns are much less brittle in this sense.
Next we explain why Matérns are different.

\section{Smoothness matters} \label{sec:smoothness-matern}

Many commonly used kernels are stationary, which means that $k(x, y) = \Phi(x - y)$ for a function $\Phi \colon \R^d \to \R$ and all $x, y \in \R^d$.
A stationary kernel is strictly positive-definite if the Fourier transform of $\Phi$, or the \emph{spectral density},
\begin{equation*}
  (\mathcal{F} \Phi)(\omega) = \frac{1}{(2\pi)^{d/2}} \int_{\R^d} \Phi(x) e^{-i x \cdot \omega} \dif x 
\end{equation*}
exists and is everywhere positive.
The Fourier transform is real because $\Phi$ is symmetric about the origin.
The rate of decay of the spectral density determines the smoothness of the kernel.
In the univariate case,
\begin{equation*}
  \Phi^{(2m)}(0) = \frac{1}{\sqrt{2\pi}} (-1)^m \int_{\R} \omega^{2m} (\mathcal{F} \Phi)(\omega) \dif \omega,
\end{equation*}
which exists only if the mapping $\omega \mapsto \omega^{2m} (\mathcal{F} \Phi)(\omega)$ is integrable.
The larger $m$ is, the faster the spectral density must tend to zero as $\lvert \omega \rvert \to \infty$ if this mapping is to be integrable.
That is, the more differentiable one wishes the kernel to be, the faster its spectral density must decay.
The spectral density 
\begin{equation*}
  (\mathcal{F}\Phi)(\omega) = \ell^d \exp(-\ell^2 \lVert \omega \rVert_2^2 / 2 )
\end{equation*}
of the Gaussian kernel~\eqref{eq:gaussian} is Gaussian and thus decays faster than any polynomial, meaning that the Gaussian kernel is infinitely differentiable (a fact for which it is of course wholly unnecessary to invoke Fourier analysis).

\subsection{Matérns and finite smoothness}

In contrast to the Gaussian, the spectral density of a Matérn kernel 
\begin{equation*}
  k_\nu(x, y) = \Phi_\nu(x - y) = \frac{2^{1-\nu}}{\Gamma(\nu)} \bigg( \frac{\sqrt{2\nu} \norm[0]{x - y}_2}{\ell} \bigg)^\nu \mathcal{K}_\nu \bigg( \frac{\sqrt{2\nu} \norm[0]{ x - y }_2}{\ell} \bigg)
\end{equation*}
on $\R^d$ has polynomial decay~\citep[Theorem~6.13]{Wendland2005}:
\begin{equation} \label{eq:matern-spectral}
  (\mathcal{F} \Phi_\nu)(\omega) = \frac{2^{d/2} \Gamma(\nu+d/2)}{\Gamma(\nu)} \bigg( \frac{2\nu}{\ell^2} \bigg)^\nu
\bigg(\frac{2\nu}{\ell^2}+ \lVert \omega \rVert_2^2\bigg)^{-(\nu+d/2)} .
\end{equation}
There is a sharp contrast between the behaviour of the conditional variance for the Gaussian kernel and kernels whose spectral density decays polynomially.

\begin{theorem} \label{thm:matern}
  Consider a stationary strictly positive-definite kernel $k(x, y) = \Phi(x - y)$ for a continuous and integrable $\Phi \colon \R^d \to \R$. Suppose that there are $c > 0$ and $s > 0$ such that $(\mathcal{F}\Phi)(\omega) \geq c (1 + \lVert \omega \rVert_2^2)^{-s}$ for all $\omega \in \R^d$.
  Let $\sigma \geq 0$ and $X_n = \{x_1, \ldots, x_n \} \subset \R^d$.
  If $\delta(x, X_n) = \min_{x_i \in X_n} \lVert x - x_i \rVert_2 \leq 1$, then
  \begin{equation} \label{eq:matern-lower-bound}
    \V_\sigma(x \mid \mathcal{D}_n) \geq C \delta(x, X_n)^{2s - d},
  \end{equation}
  where $C > 0$ is a constant that does not depend on $\sigma$, $x$, or $X_n$.
\end{theorem}
\begin{proof}
  See \Cref{sec:proofs-smoothness}.
\end{proof}

The quantity $\delta(x, X_n)$ appearing in~\eqref{eq:matern-lower-bound} is the distance from $x$ to the nearest point in $X_n$.
When $s > d/2$, a condition that any valid Matérn kernel satisfies, this theorem states that the variance at a point $x$ can tend to zero as $n \to \infty$ only if there are input points arbitrarily close to $x$. 
This is in stark contrast to \Cref{thm:interpolation-1d,thm:interpolation} in which the variance for the Gaussian kernel tends to zero (and does so extremely fast) even if the input points are far away from $x$.
These theorems make no reference to $\delta(x, X_n)$ or any other notions of closeness of $x$ to $X_n$.
The fundamental reason for this drastic difference is that the Gaussian kernel is analytic while Matérns are not.

\pgfplotsset{colormap={graybg}{rgb(0cm)=(0.6,0.6,0.6); rgb(1cm)=(0,0,0)}}
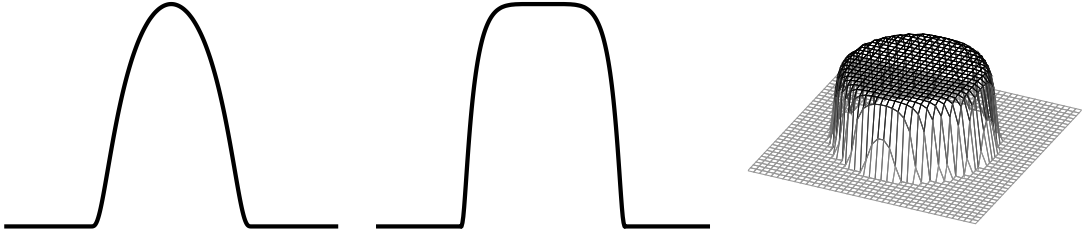
\begin{figure}
\centering
\begin{tikzpicture}
\begin{groupplot}[
    group style={group size=3 by 1, horizontal sep=0.5cm},
    width=6cm, height=6cm,
]

\nextgroupplot[
    hide axis,
    xmin=-2, xmax=2,
    ymin=-0.2, ymax=1.3,
    samples=200,
    domain=-2:2,
    smooth,
]
\addplot[ultra thick, black] {(abs(x)<1) * exp(-1/(1-x^2))*2.71828};

\nextgroupplot[
    hide axis,
    xmin=-2, xmax=2,
    ymin=-0.2, ymax=1.3,
    samples=200,
    domain=-2:2,
    smooth,
]
\addplot[ultra thick, black] {(abs(x)<1) * exp(-1/(1-x^6))*2.71828};

\nextgroupplot[
    hide axis,
    domain=-1.5:1.5,
    y domain=-1.5:1.5,
    samples=40,
    samples y=40,
    axis equal image,
    colormap name=graybg,
]
\addplot3[mesh, thin, shader=interp] {(x^2+y^2<1) * exp(-1/(1-(x^2+y^2)^12))*2.71828};

\end{groupplot}
\end{tikzpicture}
\caption{Infinitely differentiable but non-analytic bump functions on $\R$ and $\R^2$.}
\label{fig:bumps}
\end{figure}

\subsection{Analytic functions and kernels} \label{sec:analytic}

An infinitely differentiable function $f \colon D \to \R$ on an open set $D \subset \R^d$ is \emph{real analytic} if for every $x_0 \in D$ there is a neighbourhood $A \subset D$ of $x_0$ such that the Taylor series
\begin{equation*}
  T_{x_0}(x) = \sum_{\alpha \in \N_0^d} \frac{\mathrm{D}^\alpha f(x_0)}{\alpha!} (x - x_0)^\alpha
\end{equation*}
converges to $f(x)$ for every $x \in A$.
We use standard multi-index notations $x^\alpha = x_1^{\alpha_1} \cdots x_d^{\alpha_d}$, $\alpha! = \alpha_1! \cdots \alpha_d!$, and $\mathrm{D}^\alpha f(x) = \partial_1^{\alpha_1} \cdots \partial_d^{\alpha_d} f(x)$.
Every real analytic function is infinitely differentiable (for otherwise the Taylor series could not be defined) but an infinitely differentiable function need not be real analytic.
The standard example of such a function is the bump function given by 
\begin{equation*}
  f(x) = 
  \begin{dcases}
    \exp\bigg( \! -\frac{1}{1-\lVert x \rVert_2^2} \bigg) &\text{ if } \quad \lVert x \rVert_2 < 1, \\
    0 &\text{ if } \quad \lVert x \rVert_2 \geq 1 .
  \end{dcases}
\end{equation*}
\Cref{fig:bumps} shows other examples of infinitely differentiable yet non-analytic functions.
An infinitely differentiable function is real analytic if and only if $\mathrm{D}^\alpha f(x)$ does not grow too fast as $\alpha$ increases.
Namely, if for every compact $K \subset D$ there is $C \geq 0$ such that 
\begin{equation} \label{eq:analytic-derivatives}
  \sup_{x \in K} \lvert \mathrm{D}^\alpha f(x) \rvert \leq C^{\lvert \alpha \rvert + 1} \alpha! \quad \text{ for every } \quad \alpha \in \N_0^d,
\end{equation}
then $f$ is real analytic and vice versa.
A crucial property of non-trivial real analytic functions is that they cannot vanish on large sets.
See \citet[Cor.\@~1.2.7]{KrantzParks2002} and \citet{Mityagin2020} for the following theorem.\footnote{We give a version of the theorem that is as straightforward as possible. A more general version states that it suffices for $f$ to vanish on a sequence that has an accumulation point if $d = 1$ and on a set with positive Lebesgue measure if $d \geq 2$.}

\begin{theorem}[Identity theorem] \label{thm:identity-theorem}
Let $f \colon D \to \R$ be a real-analytic function on an open and connected set $D \subset \R^d$. 
If $f$ vanishes on an open subset of $D$, then $f = 0$ on $D$.
\end{theorem}

An immediate corollary of the identity theorem is that two real analytic functions that coincide on an open and connected set must coincide everywhere.
In other words, a real analytic function is fully determined by its values on any given open set. 
This explains the observations in \Cref{sec:main-sec}: With the Gaussian kernel the latent function is modelled as a real analytic function and thus having observations in any given open subset suffices to fully learn it.
All analytic kernels induce similar behaviour.

A square-integrable function $f$ is real analytic if its Fourier transform tends to zero exponentially fast, meaning that there is $a > 0$ such that $(\mathcal{F}f)(\omega) = O(\exp(-a \lVert \omega \rVert_2))$.
This is easy to verify with the help of the Fourier identity $(\mathcal{F} \, \mathrm{D}^\alpha f)(\omega) = (i\omega)^\alpha (\mathcal{F} f)(\omega)$ and~\eqref{eq:analytic-derivatives}.
A more general version of this result is known as the Paley--Wiener theorem~\citep[Thm.\@~IX.13]{ReedSimon1975}.
For stationary kernels we obtain the following theorem, which states that the conditional variance tends to zero everywhere essentially regardless of the distribution of the input points if the spectral density decays exponentially.

\begin{theorem} \label{thm:analytic}
  Consider a stationary strictly positive-definite kernel $k(x, y) = \Phi(x - y)$ for a continuous and integrable $\Phi \colon \R^d \to \R$.
  Suppose that there are $C \geq 0$ and $a > 0$ such that $(\mathcal{F}\Phi)(\omega) \leq C \exp(-a \lVert \omega \rVert_2)$ for all $\omega \in \R^d$.
  Let $\sigma \geq 0$.
  If $X_n = \{x_1, \ldots, x_n \}$ for a sequence $\{x_i\}_{i=1}^\infty \subset \R^d$ of pairwise distinct points whose closure has non-empty interior, then
  \begin{equation*}
     \sup_{x \in K} \V_\sigma(x \mid \mathcal{D}_n) \to 0 \quad \text{ as } \quad n \to \infty \quad \text{ for any compact } \quad K \subset \R^d.
  \end{equation*}
\end{theorem}
\begin{proof}
  See \Cref{sec:proofs-smoothness}.
\end{proof}

See \cite{SunZhou2008} for related results.
\Cref{thm:analytic} is a generalisation of \Cref{thm:regression}, a corresponding result for the Gaussian kernel.
The message of \Cref{thm:matern,thm:analytic} is that the rate of decay of the spectral density determines whether or not a Gaussian process model is liable to overconfidence resulting from defectively distributed input points.
If the spectral density does not decay faster than some polynomial (e.g., Matérns), the conditional variance cannot tend to zero everywhere unless the input points are space-filling; if the spectral density decays at least exponentially, the variance tends to zero whenever the inputs cover \emph{some} open set.

There are well-known kernels that lie on the ``boundary'' between the two cases.
An \emph{inverse multiquadric kernel}
\begin{equation} \label{eq:multiquadric}
  k(x, y) = \bigg(1 + \frac{\lVert x - y \rVert_2^2}{\ell^2} \bigg)^{-\beta/2} 
\end{equation}
is strictly positive-definite and integrable on $\R^d$ whenever $\beta > d$~\citep[Thm.\@~6.13]{Wendland2005}.
The special case $\beta = 2$ gives rise to the \emph{inverse quadratic kernel}
\begin{equation} \label{eq:inverse-quadratic}
  k(x, y) = \bigg(1 + \frac{\lVert x - y \rVert_2^2}{\ell^2} \bigg)^{-1} . 
\end{equation}
From~\eqref{eq:matern-spectral} we see that the spectral density of a Matérn is an inverse multiquadric, so by Fourier inversion the spectral density of an inverse multiquadric is a Matérn.
It is then straightforward to work out that the spectral density of an inverse multiquadric is
\begin{equation*}
  (\mathcal{F} \Phi)(\omega) = \ell^d \frac{2^{1-\beta/2}}{\Gamma(\beta/2)} (\ell \lVert \omega \rVert_2)^{(\beta-d)/2} \mathcal{K}_{(\beta-d)/2}(\ell \lVert \omega \rVert_2) .
\end{equation*}
Because the Bessel function has the asymptotics $\mathcal{K}_\nu(r) = \Theta(r^{-1/2} e^{-r})$ as $r \to \infty$ for any $\nu > 0$~\citep[Sec.\@~5.1]{Wendland2005}, the spectral density satisfies 
\begin{equation*}
  (\mathcal{F}\Phi)(\omega) = \Theta(\lVert \omega \rVert_2^{(\beta-d)/2-1/2} \exp(-\ell \lVert \omega \rVert_2)) .
\end{equation*}
The spectral density decays exponentially fast and the kernel falls under \Cref{thm:analytic}.

\subsection{Infinitely smooth non-analytic kernels} \label{sec:gevrey}

\Cref{thm:matern} covers finitely smooth kernels and \Cref{thm:analytic} real analytic kernels.
\emph{Gevrey kernels} are infinitely differentiable but non-analytic and therefore fall between the two classes.
To the best of our knowledge these kernels have not appeared before in the Gaussian process literature.
We say that a stationary kernel on $\R^d$ is a Gevrey kernel of order $\eta \in (0, 1)$ if its spectral density is
\begin{equation*}
  (\mathcal{F} \Phi_\eta)(\omega) = C_\eta \exp(- \rho \lVert \omega \rVert_2^\eta),
\end{equation*}
where the constant $C_\eta = (2\pi)^{d/2} / \int_{\R^d} \exp(- \rho \lVert \omega \rVert_2^\eta) \dif \omega$ is there to ensure that $\Phi_\eta(0) = 1$.
Note that, formally, $\eta = 1$ corresponds to an inverse multiquadric with $\beta = d + 1$ and $\eta = 2$ to the Gaussian kernel.
Because spectral densities of Gevrey kernels decay faster than any polynomial, these kernels are infinitely differentiable.
However, Gevrey kernels are not analytic.
One can verify this for example by studying how fast 
\begin{equation*}
  \Phi_\eta^{(2\alpha)}(0) = \frac{C_\eta}{(2\pi)^{d/2}} (-1)^{\lvert \alpha \rvert} \int_{\R^d} \omega^{2\alpha} \exp(-\rho \lVert \omega \rVert_2^\eta) \dif \omega
\end{equation*}
blows up as $\alpha$ increases and using the analyticity characterisation in~\eqref{eq:analytic-derivatives}.
Gevrey kernels are closely related to Gevrey spaces of infinitely differentiable but non-analytic functions~\citep{Rodino1993} that appear in the analysis of partial differential equations. 
The following theorem shows that Gevrey kernels resemble Matérns in that the input points must be space-filling if the conditional variance is to vanish.

\begin{theorem} \label{thm:gevrey}
  Consider a Gevrey kernel with $\eta \in (0, 1)$. 
  Let $\sigma \geq 0$ and $X_n = \{x_1, \ldots, x_n \} \subset \R^d$.
  If $\delta(x, X_n) = \min_{x_i \in X_n} \lVert x - x_i \rVert_2$ is sufficiently small, then
  \begin{equation*}
    \V_\sigma(x \mid \mathcal{D}_n) \geq C \exp(- c \,\delta(x, X_n)^{-\kappa}) ,
  \end{equation*} 
  where $C, c, \kappa > 0$ are constants that do not depend on $\sigma$, $x$, or $X_n$.
\end{theorem}
\begin{proof}
  See \Cref{sec:proofs-smoothness}.
\end{proof}

Unfortunately, not many Gevrey kernels are available in closed form. 
A few somewhat tractable special cases exist because it turns out that Gevrey kernels correspond to density functions of stable distributions~\citep{Zolotarev1986, GaroniFrankel2002}.
Let $d = 1$.
The simplest special case is $\eta = 2/3$.
The resulting \emph{Whittaker kernel} is defined by
\begin{equation*}
  \Phi_{\eta=2/3}(r) = \frac{2}{3\sqrt{3} \, \gamma \lvert r \rvert} \exp\bigg(\frac{2}{27 \gamma^2 r^2} \bigg) \mathcal{W}_{-1/2, 1/6} \bigg(\frac{4}{27 \gamma^2 r^2} \bigg) .
\end{equation*}
Here $\gamma = \rho^{-1/\eta}$ and $\mathcal{W}_{\kappa, \mu}$ is the Whittaker function, which has the expression
\begin{equation*}
  \mathcal{W}_{\kappa, \mu}(z) = \frac{z^\kappa e^{-z/2}}{\Gamma(1/2 - \kappa + \mu)} \int_0^\infty t^{-\kappa-1/2+\mu} \bigg(1 + \frac{t}{z} \bigg)^{\kappa-1/2+\mu} e^{-t} \dif t
\end{equation*}
when $\kappa-1/2-\mu \in \C$ is not an integer and has non-positive real part.
\Cref{fig:whittaker} compares the conditional variance for the Whittaker kernel to a few popular kernels.

\begin{center}
  \begin{figure}[t]
    \includegraphics[width=\textwidth]{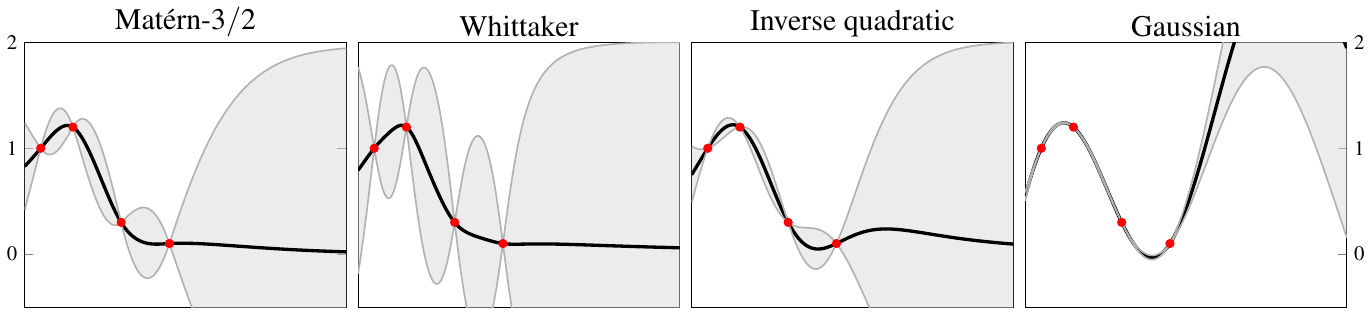}    
    \caption{The conditional mean (black) and the pointwise 95\% credible intervals (shaded) obtained from the conditional variance over the interval $[0, 1]$ for the Matérn-$3/2$, Whittaker, inverse quadratic, and Gaussian kernels with $\ell = 0.3$. The data consist of four observations located on the first half of the interval that are indicated in red. Despite being infinitely differentiable, the Whittaker and inverse quadratic kernels yield much larger variances than the Gaussian on the second half of the interval that contains no input points.}
    \label{fig:whittaker}
  \end{figure}
\end{center}

\section{Apology of the Gaussian kernel} \label{sec:apology}

Finally, it seems fitting to review some of the reasons, \emph{post hoc} or not, that have made the Gaussian kernel popular, or at the very least may have played a part in keeping it popular.

\subsection{The Gaussian kernel is Gaussian}

No doubt the main cause for the popularity of the Gaussian kernel is the fact that it is nothing but the Gaussian function that every student is familiar with.
Its algebraic form is simple and convenient.
No Bessel functions or other objects more advanced than those encountered on basic calculus courses are required for implementation and superficial understanding.
Moreover, because the Gaussian function is infinitely differentiable (and differentiating it is relatively easy) and decays very fast, one never has to worry about differentiability, integrability, or other such pesky technicalities.
As its derivatives are straightforward to compute, the Gaussian kernel can be easily applied to modelling with derivative data.
That many integrals involving the Gaussian kernel can be computed explicitly makes it a convenient kernel to use when working with Hilbert space embeddings of distributions~\citep{Muandet2017}.

\subsection{The only factorisable isotropic kernel} \label{sec:factorisable}

A kernel on $\R^d$ is \emph{isotropic} if there is a function $\phi \colon [0, \infty) \to \R$ such that $k(x, y) = \phi(\lVert x - y \rVert_2)$ for all $x, y \in \R^d$.
That is, an isotropic kernel only depends on the distance between the inputs.
Often one has to decide whether to use an isotropic kernel or a \emph{product-isotropic} kernel
\begin{equation*}
  k(x, y) = k_1(x_1, y_1) \times \cdots \times k_d(x_d, y_d) = \phi_1(\lvert x_1 - y_1 \rvert) \times \cdots \times \phi_d(\lvert x_d - y_d \rvert),
\end{equation*}
where $k_i(x_i, y_i) = \phi_i(\lvert x_i - y_i \rvert)$ are isotropic kernels on $\R$.
When using the Gaussian kernel one does not have to decide, for it factorises as a product of its one-dimensional versions:
\begin{equation*}
  k(x, y) = \exp\bigg( \! - \frac{ \norm[0]{x - y}_2^2}{2 \ell^2} \bigg) = \exp\bigg( \! - \frac{ (x_1-y_1)^2}{2 \ell^2} \bigg) \times \cdots \times \exp\bigg( \! - \frac{ (x_d-y_d)^2}{2 \ell^2} \bigg) .
\end{equation*}
It would be trivial to include and learn a different lengthscale for each dimension, which is sometimes referred to as automatic relevance determination.
This property is unique to the Gaussian kernel as under very mild assumptions there are no other isotropic kernels that are coordinate-wise products of their one-dimensional versions~\citep[pp.\@~55--56]{Stein1999}.\footnote{This follows from a characterisation of the exponential function as the only function $f \colon \R \to \R$ satisfying mild regularity properties and $f(1) = e$ that has the multiplicative property $f(x + y) = f(x) f(y)$; see Exercise~18.46 in \citet{Hewitt1965}.} 

\subsection{Simple series expansion}

Access to a series expansion of the form $k(x, y) = \sum_{p=0}^\infty \psi_p(x) \psi_p(y)$ is useful, be it to construct low-rank approximations by truncating the expansion~\citep{SolinSarkka2020} or for the purposes of theory~\citep{Minh2010}.
The Gaussian kernel has the simplest series expansion among commonly used stationary kernels, trivially derived from the Taylor expansion of the exponential function (here we set $d=1$ for simplicity):
\begin{equation} \label{eq:gaussian-expansion}
  \begin{split}
    k(x, y) = \exp\bigg( \! - \frac{ (x - y)^2}{2 \ell^2} \bigg) &= e^{-x^2/(2\ell^2)} e^{x y / \ell^2} e^{-y^2/(2\ell^2)} \\
    &= \sum_{p=0}^\infty \bigg( \frac{x^p}{\ell^p \sqrt{ p!}} \, e^{-x^2/(2\ell^2)} \bigg) \bigg( \frac{y^p}{\ell^p \sqrt{ p!}} \, e^{-y^2/(2\ell^2)} \bigg) .
  \end{split}
\end{equation}
The Gaussian also has a known Mercer expansion with respect to the normal distribution.
Though the Mercer expansion is more complicated than the expansion in~\eqref{eq:gaussian-expansion}, it only requires Hermite polynomials and exponentials.
Series expansions, Mercer or not, for other stationary kernels tend to be much more involved; examples can be found in \cite{TronarpKarvonen2024}.

\subsection{Limiting case in the Matérn class} \label{sec:matern-gaussian-limit}

The parameter $\nu > 0$ of the Matérn class
\begin{equation*}
  k_\nu(x, y) = \frac{2^{1-\nu}}{\Gamma(\nu)} \bigg( \frac{\sqrt{2\nu} \norm[0]{x - y}_2}{\ell} \bigg)^\nu \mathcal{K}_\nu \bigg( \frac{\sqrt{2\nu} \norm[0]{ x - y }_2}{\ell} \bigg)
\end{equation*}
controls the differentiability of the kernel and the corresponding Gaussian process.
It is well known that $k_\nu$ tends pointwise to the Gaussian kernel as $\nu \to \infty$~\citep[pp.\@~49--50]{Stein1999}.\footnote{There are other natural classes of kernels that converge to a Gaussian as their smoothness increases~\citep{Chernih2014}.}
It tends to be difficult to estimate differentiability from finite data.
In using the Gaussian kernel one in essence removes $\nu$, a difficult-to-specify parameter, from the Matérn model.
As discussed in \Cref{sec:factorisable}, this also removes the need to decide between an isotropic and product-isotropic kernel.
Because Matérns encode a full scale of Sobolev regularity assumptions, the Gaussian can be thought of as an ``infinitely smooth Sobolev prior'' (this statement is not fully devoid of mathematical meaning; see \citealp{NovakUllrich2018}).

\subsection{Member of the $\gamma$-exponential class}

Let $\gamma \in (0, 2]$.
Then the \emph{$\gamma$-exponential} kernel
\begin{equation} \label{eq:gamma-exponential}
  k_\gamma(x, y) = \exp\bigg(\! -\frac{\lVert x - y \rVert_2^\gamma}{\gamma \, \ell^\gamma }\bigg)
\end{equation}
is positive-definite.
These kernels were popularised by \citet{Sacks1989} and often appear in emulation literature.\footnote{We also highlight an observation by \citet[p.\@~413]{Sacks1989}: ``The case $p = 2$ gives a process with infinitely differentiable paths (mean square sense) and is useful when the response is analytic.''}
With parameters $\gamma = 2$ and $\gamma = 1$ we get the Gaussian and the Matérn-$1/2$ kernels $k_2(x, y) = \exp(-\lVert x - y \rVert_2^2/(2\ell^2))$ and $k_1(x, y) = \exp(-\lVert x - y \rVert_2 / \ell)$.
The $\gamma$-exponential therefore provides another ``natural'' connection between Gaussians and finitely smooth kernels.
However, when one looks any deeper it quickly becomes apparent that the connection is far from natural: For every $\gamma \in (0, 2)$ the $\gamma$-exponential is at most once differentiable with respect to both arguments, yet becomes infinitely differentiable when $\gamma = 2$.
That is, the class of kernels defined by~\eqref{eq:gamma-exponential} is continuous in $\gamma$ when it comes to the functional form of the kernel, but not when it comes to the properties of the kernel.
The only advantage the $\gamma$-exponential class enjoys over Matérns in~\eqref{eq:matern} is its simplicity (no Bessel functions).

\subsection{Universality}

One factor that may have led people to use the Gaussian kernel is that it is \emph{universal}. 
A kernel on a compact $D \subset \R^d$ is universal if its reproducing kernel Hilbert space (see \Cref{sec:rkhs}) is dense in $C(D)$, the space of continuous functions on $D$.
A universal kernel can therefore be used to approximate any continuous function, and this is a desirable property when no strong structural information is available on the unknown function being modelled.
However, universality has little use as a kernel decision criterion among kernels in everyday usage because every commonly used kernel is universal; for example, all stationary kernels that appear in this article are universal~\citep[Thm.\@~17]{Micchelli2006}.

\subsection{When to use the Gaussian}

The central message of this article is that the Gaussian kernel is not the kernel to default to.
It should be used only if there are compelling reasons to believe that the latent phenomenon one models is extremely smooth. 
And even then it may be safer to use other infinitely smooth kernels, such as inverse multiquadrics or Gevreys.
The Gaussian kernel is far less dangerous in regression ($\sigma > 0$) than in interpolation ($\sigma = 0$), where it can induce catastrophic failure of both numerics and uncertainty quantification.
Nevertheless, even in regression one should never intentionally use a model that is known to be further from the truth than almost any other standard model.
It should be noted that we have focussed on Gaussian processes, a perspective we admit is quite narrow.
There are related tasks in which the Gaussian kernel is known to shine.
For example, the support vector machine based on the Gaussian kernel is rate-adaptive over Sobolev spaces: It achieves (up to a sub-polynomial term) the minimax optimal learning rate $n^{-2s/(2s+d)}$ in a Sobolev space of order $s$ on a subset of $\R^d$ if the lengthscale is selected in an appropriate data-driven way~\citep{EbertsSteinwart2013}.

\section{What about kernel learning?} \label{sec:kernel-learning}

We have but touched upon kernel learning, the data-driven choice of the kernel.
Tuning a few parameters of a standard parametric family of kernels (e.g., the lengthscale as discussed in Remark~\ref{remark:parameter-estimation}) is often too inflexible.
Composing functions and kernels~\citep{lloyd2014automatic, Wilson2016} and learning a kernel in the spectral domain~\citep{lazaro2010sparse, remes2017non} are the two main approaches to flexible kernel learning.
Our results as such do not apply to kernel learning.
However, our work opens an interesting research direction to inform the design of kernel learning architectures.
Let us consider deep kernel learning~\citep{Wilson2016} as a concrete example.
In deep kernel learning, flexibility is achieved by using a composite kernel
\begin{equation*}
  k(x, y) = k_0( g(x, w), g(y, w) ),
\end{equation*}
where $k_0$ is a potentially parametric base kernel and $g$ a neural network with weights or other parameters $w$.
The base kernel $k_0$ is often Gaussian, and data are used to learn the weights and other parameters.
As compositions of analytic functions are analytic, our work, when combined with the well-understood properties of composite kernels~\citep[Ch.\@~5]{Paulsen2016}, suggests that it may be best to avoid analytic activation functions (e.g., the swish function) in the deep kernel learning framework of \citet{Wilson2016}.

\section{Proofs} \label{sec:proofs}

This section contains the proofs of the theorems appearing in this article.

\subsection{Reproducing kernel Hilbert spaces} \label{sec:rkhs}

Every positive-semidefinite kernel $k \colon D \times D \to \R$ on a set $D$ induces a unique \emph{reproducing kernel Hilbert space} (RKHS).
This Hilbert space, denoted $H(k)$, consists of certain functions $f \colon D \to \R$ and is characterised by $k(\cdot, x) \in H(k)$ for every $x \in D$ and the important \emph{reproducing property}:
\begin{equation} \label{eq:reproducing-property}
  f(x) = \langle f, k(\cdot, x) \rangle_{H(k)} \quad \text{ for all } \quad f \in H(k) \: \text{ and } \: x \in D.
\end{equation}
We refer to \cite{Berlinet2004} and \cite{Paulsen2016} for the theory of RKHSs.
From the functional form of the kernel alone it tends to be difficult to deduce what $H(k)$ and its inner product look like.
Fortunately, if the kernel is stationary on $\R^d$, meaning that $k(x, y) = \Phi(x - y)$ for $\Phi \colon \R^d \to \R$, and $\Phi$ is continuous and integrable, the RKHS admits the convenient spectral characterisation
\begin{equation} \label{eq:rkhs-spectral}
  H(k) = \Set[\bigg]{f \in L^2(\R^d)}{ \lVert f \rVert_{H(k)}^2 = \frac{1}{(2\pi)^{d/2}} \int_{\R^d} \frac{\lvert (\mathcal{F}f)(\omega) \rvert^2}{(\mathcal{F} \Phi)(\omega)} \dif \omega < \infty}
\end{equation}
in terms of the spectral density, $\mathcal{F} \Phi$ (recall \Cref{sec:smoothness-matern}).
See Theorem~10.12 in \citet{Wendland2005} or \citet{KimeldorfWahba1970} for the spectral characterisation.
The faster the spectral density tends to zero (i.e., the smoother the kernel), the smaller the RKHS.

The proofs are based on the following equivalence between the conditional variance and worst-case error. 
More such equivalences can be found in Chapters~5 and~6 of \cite{Kanagawa2018}.
We write $k^x = k(\cdot, x)$ for a kernel translate.

\begin{lemma} \label{lemma:worst-case}
  Let $\sigma \geq 0$.
  Then
  \begin{equation} \label{eq:worst-case}
    \V_\sigma(x \mid \mathcal{D}_n) = \min_{\b{u} \in \R^n} \bigg\{ \bigg\lVert k^x - \sum_{i=1}^n u_i k^{x_i} \bigg\rVert_{H(k)}^2 + \sigma^2 \norm[0]{ \b{u} }_2^2 \bigg\}
  \end{equation}
  for every $x \in D$.
  Moreover, $\V_\sigma(x \mid \mathcal{D}_{n+1}) \leq \V_\sigma(x \mid \mathcal{D}_n)$ for every $n \geq 1$ and $x \in D$.
\end{lemma}
\begin{proof}
  By the reproducing property~\eqref{eq:reproducing-property}, $\langle k^x, k^y \rangle_{H(k)} = k(x, y)$, so that expanding the square in~\eqref{eq:worst-case} gives the quadratic objective
  \begin{equation*}
    J_n(\b{u}) = k(x, x) - 2 \b{u}^\T \b{k}_n(x) + \b{u}^\T \b{K}_n \b{u} + \sigma^2 \lVert \b{u} \rVert_2^2 = k(x,x) - 2 \b{u}^\T \b{k}_n(x) + \b{u}^\T ( \b{K}_n + \sigma^2 \b{I}_n) \b{u} .
  \end{equation*}
  Because $\b{K}_n + \sigma^2 \b{I}_n$ is a positive-definite matrix, $J_n$ is strictly convex and its unique minimiser is $\b{u}_* = ( \b{K}_n + \sigma^2 \b{I}_n)^{-1} \b{k}_n(x)$, for which
  \begin{equation*}
    J_n(\b{u}_*) = k(x, x) - \b{k}_n(x)^\T ( \b{K}_n + \sigma^2 \b{I}_n)^{-1} \b{k}_n(x) = \V_\sigma(x \mid \mathcal{D}_n)
  \end{equation*}
  by~\eqref{eq:gp-regression-var}.
  Monotonicity in $n$ follows from $J_n(\b{u}) = J_{n+1}((\b{u},0))$.
\end{proof}

The reproducing property can be utilised to rewrite the variance as a worst-case error.

\begin{lemma} \label{lemma:worst-case-2}
  Let $\sigma \geq 0$.
  Then
\begin{equation} \label{eq:worst-case-2}
    \V_\sigma(x \mid \mathcal{D}_n) = \min_{\b{u} \in \R^n} \bigg\{ \sup_{ \norm[0]{f}_{H(k)} \leq 1} \, \bigg\lvert f(x) - \sum_{i=1}^n f(x_i) u_i \bigg\rvert^2 + \sigma^2 \lVert \b{u} \rVert_2^2 \bigg\}.
\end{equation}
\end{lemma}
\begin{proof}
  That the right-hand side in~\eqref{eq:worst-case-2} does not exceed that in~\eqref{eq:worst-case} follows directly from the reproducing property and the Cauchy--Schwarz inequality.
  That we obtain an equality is verified by considering the function $f = (k^x - \sum_{i=1}^n u_i k^{x_i})/\lVert k^x - \sum_{i=1}^n u_i k^{x_i} \rVert_{H(k)}$, which has unit norm in $H(k)$.
  For details, see for example Sections~5.4 and~6.6 in \cite{Kanagawa2018}. 
\end{proof}

The following results are also useful.

\begin{lemma} \label{lemma:sigma-increasing}
  If $\sigma \geq \varsigma \geq 0$, then $\V_\sigma(x \mid \mathcal{D}_n) \geq \V_\varsigma(x \mid \mathcal{D}_n)$ for every $x \in D$.
\end{lemma}
\begin{proof}
    Because $\b{K}_n + \varsigma^2 \b{I}_n \leq \b{K}_n + \sigma^2 \b{I}_n$ if $\sigma \geq \varsigma$ in the Loewner ordering of positive-semidefinite matrices, the claim follows from the expression for the variance in~\eqref{eq:gp-regression-var}.
\end{proof}

In the next proposition, which appears as Corollary~4.36 in \cite{Steinwart2008}, we use the notation
\begin{equation*}
  k^{\alpha,\beta}(x, y) = \mathrm{D}_v^\alpha \mathrm{D}_w^\beta k(v,w) \bigl|_{\substack{v = x \\ w = y}}
\end{equation*}
for kernel derivatives. Here the subscripts denote the variable with respect to which the kernel is being differentiated.

\begin{proposition} \label{prop:derivative-repro}
  If $k$ is a positive-semidefinite kernel on an open set $D \subset \R^d$ such that $k^{\alpha,\alpha}$ exists and is continuous for every $\alpha \in \N_0^d$ such that $\lvert \alpha \rvert \leq m$, then every $f \in H(k)$ is $m$ times continuously differentiable and
  \begin{equation} \label{eq:derivative-repro} 
    \mathrm{D}^\alpha f(x) = \langle f, \mathrm{D}_x^\alpha k(\cdot, x) \rangle_{H(k)} 
  \end{equation}
  for any $\alpha \in \N_0^d$ such that $\lvert \alpha \rvert \leq m$ and $x \in D$.
\end{proposition}

\subsection{Proofs for \Cref{sec:main-sec}} \label{sec:proofs-main-sec}

This section contains the proofs for \Cref{sec:main-sec}.
\\

\noindent \textbf{Proof of \Cref{thm:interpolation-1d}}
Let
\begin{equation} \label{eq:polynomial-interpolant}
  (P_nf)(x) = \sum_{i=1}^n f(x_i) \prod_{j \neq i} \frac{x - x_j}{x_i - x_j}
\end{equation}
be the unique polynomial of degree at most $n-1$ that interpolates $f$ at our pairwise distinct points $x_1, \ldots, x_n \in D = [a, b]$.
It is a basic result of polynomial interpolation that for each $x \in D$ there is $\xi_x \in D$ such that
\begin{equation*}
  f(x) - (P_nf)(x) = \frac{f^{(n)}(\xi_x)}{n!} \prod_{i=1}^n (x - x_i)
\end{equation*}
if $f$ is $n$ times continuously differentiable.
Consequently,
\begin{equation*}
  \abs[0]{ f(x) - (P_nf)(x) } \leq \frac{\sup_{x \in D} \abs[0]{f^{(n)}(x)}}{n!} \sup_{x \in D} \, \abs[3]{ \prod_{i=1}^n (x - x_i) } \leq \frac{\sup_{x \in D} \abs[0]{f^{(n)}(x)}}{n!} (b-a)^n
\end{equation*}
for every $x \in D$.
Every element of the RKHS of the Gaussian kernel is infinitely differentiable by Proposition~\ref{prop:derivative-repro}.
Recall that we consider the case $\sigma = 0$.
From~\eqref{eq:worst-case} and the form of the polynomial interpolant in~\eqref{eq:polynomial-interpolant} it follows that
\begin{equation} \label{eq:1d-proof-variance-intermediate}
  \begin{split}
    \sup_{x \in D} \V( x \mid \mathcal{D}_n) &\leq \sup_{x \in D} \sup_{ \norm[0]{f}_{H(k)} \leq 1} \lvert f(x) - (P_nf)(x) \rvert^2 \\
    &\leq \bigg( \frac{(b-a)^n}{n!} \bigg)^2 \sup_{\norm[0]{f}_{H(k)} \leq 1} \, \sup_{x \in D} \, \abs[0]{f^{(n)}(x)}^2.
    \end{split}
\end{equation}
It is therefore necessary to bound the $n$th derivative of a function in $H(k)$.
Applying the Cauchy--Schwarz inequality to~\eqref{eq:derivative-repro} gives
\begin{equation} \label{eq:rkhs-derivative-bound}
  \abs[0]{f^{(n)}(x)} = \bigg\lvert \bigg\langle f, \frac{\partial^n}{\partial x^n} k(\cdot, x) \bigg\rangle_{H(k)} \bigg\rvert \leq \norm[0]{f}_{H(k)} \norm[3]{ \frac{\partial^n}{\partial x^n} k(\cdot, x) }_{H(k)}
\end{equation}
for every $x \in D$.
Now,
\begin{equation*}
  \norm[3]{ \frac{\partial^n}{\partial x^n} k(\cdot, x) }_{H(k)}^2 = \frac{\partial^{2n}}{\partial v^n \partial w^n} k(v, w) \biggl|_{\substack{v = x \\ w = x}} = (-1)^n \frac{\dif^{\,2n}}{\dif z^{2n}} \exp\bigg(\! -\frac{z^2}{2\ell^2} \bigg) \biggl|_{z = 0},
\end{equation*}
where the first equality is valid for any sufficiently smooth kernel and the second uses the stationarity of the Gaussian kernel.
Straightforward differentiation of the Taylor series for $\exp(-z^2/(2\ell^2))$ then gives
\begin{equation*}
  (-1)^n \frac{\dif^{\,2n}}{\dif z^{2n}} \exp\bigg(\! -\frac{z^2}{2\ell^2} \bigg) \biggl|_{z = 0} = \ell^{-2n} \frac{(2n)!}{2^n n!}.
\end{equation*}
Inserting this in~\eqref{eq:rkhs-derivative-bound} gives the bound
\begin{equation*}
  \sup_{x \in D} \abs[0]{ f^{(n)}(x) } \leq \norm[0]{f}_{H(k)} \ell^{-n} \sqrt{ \frac{(2n)!}{2^n n!} }
\end{equation*}
on derivatives of elements of the Gaussian RKHS.
Inserting this bound in~\eqref{eq:1d-proof-variance-intermediate} and applying the non-asymptotic version of Stirling's approximation~\citep{Robbins1955},
\begin{equation} \label{eq:stirling}
  \sqrt{2\pi} \, n^{n + 1/2} e^{-n} \leq n! \leq \sqrt{2\pi} \, n^{n + 1/2} e^{-n+1},
\end{equation}
to $(2n)!$ and $(n!)^2$ yields
\begin{equation*} 
  \sup_{x \in D} \V( x \mid \mathcal{D}_n) \leq \bigg( \frac{b-a}{\ell} \bigg)^{2n} \frac{(2n)!}{2^n (n!)^2 } \cdot \frac{1}{n!} \leq \frac{e}{\sqrt{\pi n}} \bigg( \frac{\sqrt{2}(b-a)}{\ell} \bigg)^{2n} \frac{1}{n!} .
\end{equation*}
The claimed bound follows from $e/\sqrt{\pi n} \leq 2$.
\qed

\noindent\textbf{Proof of \Cref{thm:interpolation}}
Because $D$ is bounded, it is contained in some hypercube $D'$.
Since $\sup_{x \in D} \V(x \mid \mathcal{D}_n) \leq \sup_{x \in D'} \V(x \mid \mathcal{D}_n)$, to derive an upper bound we may assume that $D = [a, b]^d$ for some $a < b$.
Since the conditional variance cannot increase when new points are added, we may assume that $n = m^d$ and $X_n = X_1' \times \cdots \times X_d'$ for $X_j' = \{x_{j,1}, \ldots, x_{j,m}\} \subset \R$ such that $X_n \subset D$.
Write $x = (x_1, \ldots, x_d) \in \R^d$ and $y = (y_1, \ldots, y_d) \in \R^d$.
Because both the multivariate Gaussian kernel
\begin{equation*}
  k(x, y) = \exp\bigg( \! - \frac{ \norm[0]{x - y}_2^2}{2 \ell^2} \bigg) = \prod_{j=1}^d r(x_j, y_j) \coloneqq \prod_{j=1}^d \exp\bigg( \! - \frac{ (x_j - y_j)^2}{2 \ell^2} \bigg)
\end{equation*}
and the points $X_n$ have tensor product structure, we have
\begin{equation*}
  \b{K}_n = \b{R}_{1,m} \otimes \cdots \otimes \b{R}_{d,m} \in \R^{n \times n } \quad \text{ and } \quad \b{k}_n(x) = \b{r}_{1,m}(x_1) \otimes \cdots \otimes \b{r}_{d,m}(x_d) \in \R^{n},
\end{equation*}
where $\otimes$ is the Kronecker product, $\b{R}_{j,m} = (r(x_{j,i}, x_{j,l}))_{i,l=1}^{m}$, and $\b{r}_{j,m}(x_j) = (r(x_j, x_{j,i}))_{i=1}^m$.
It now follows from properties of the Kronecker product that
\begin{equation*}
  \begin{split}
  \V(x \mid \mathcal{D}_{n} ) = k(x, x) - \b{k}_{n}(x)^\T \b{K}_{n}^{-1} \b{k}_n(x) &= \prod_{j=1}^d r(x_j, x_j) - \prod_{j=1}^d \b{r}_{j,m}(x_j)^\T \b{R}_{j,m}^{-1} \b{r}_{j,m}(x_j) \\
  &\eqqcolon \prod_{j=1}^d r_j - \prod_{j=1}^d q_j.
  \end{split}
\end{equation*}
Denote $r_{i:j} = r_i \times \cdots \times r_j$ and $q_{i:j} = q_i \times \cdots \times q_j$ for $1 \leq i \leq j \leq d$.
Then iteration gives
\begin{equation*}
  \begin{split}
    \V(x \mid \mathcal{D}_{n} ) = r_1 r_{2:d} - q_1 q_{2:d} = r_{2:d}( r_1 - q_1) + (r_{2:d} - q_{2:d} ) q_1 = \sum_{j=1}^d (r_j - q_j) r_{j+1:d} \prod_{i=1}^{j-1} q_i,
  \end{split}
\end{equation*}
where we use the convention $r_{d+1:d} = 1$.
Because the conditional variance is non-negative, we have
\begin{equation*}
  q_j = \b{r}_{j,m}(x_j)^\T \b{R}_{j,m}^{-1} \b{r}_{j,m}(x_j) \leq r_j = r(x_j, x_j) = 1,
\end{equation*}
from which it follows that
\begin{equation} \label{eq:var-multi-sum-expansion}
  \V(x \mid \mathcal{D}_{n} ) \leq \sum_{j=1}^d (r_j - q_j) = \sum_{j=1}^d \big[ r(x_j, x_j) - \b{r}_{j,m}(x_j)^\T \b{R}_{j,m}^{-1} \b{r}_{j,m}(x_j) \big].
\end{equation}
Because $r$ is the univariate Gaussian kernel and $X_j' \subset [a, b]$ for every $j$, we may apply \Cref{thm:interpolation-1d} with $n = m$ to each term in~\eqref{eq:var-multi-sum-expansion}. 
This yields
\begin{equation*}
  \V(x \mid \mathcal{D}_{n} ) \leq 2 \sum_{j=1}^d \bigg( \frac{\sqrt{2}(b-a)}{\ell} \bigg)^{2m} \frac{1}{m!} = 2d \bigg( \frac{\sqrt{2}(b-a)}{\ell} \bigg)^{2m} \frac{1}{m!} ,
\end{equation*}
which completes the proof.
\qed

\noindent\textbf{Proof of Proposition~\ref{prop:gaussian-lower-bound}}
By Lemma~\ref{lemma:sigma-increasing} it suffices to consider the case $\sigma = 0$.
The RKHS of the Gaussian kernel on $\R^d$ has the characterisation~\citep{Minh2010}
\begin{equation} \label{eq:minh-characterisation}
  H(k) = \Set[\bigg]{f(x) = e^{-\lVert x \rVert_2^2/(2\ell^2)} \sum_{\alpha \in \N_0^d} c_\alpha x^\alpha }{ \lVert f \rVert_{H(k)}^2 = \sum_{\alpha \in \N_0^d} \ell^{2\lvert \alpha \rvert} \alpha! c_\alpha^2 < \infty } .
\end{equation}
Because $\{x_i\}_{i=1}^\infty \subset R_0$, where $R_0$ is the set of roots of some non-trivial $d$-variate polynomial $P$ and the function \smash{$g(z) = e^{-\lVert z \rVert_2^2/(2\ell^2)} P(z)$} is in $H(k)$ and has positive $H(k)$-norm by~\eqref{eq:minh-characterisation}, we get from~\eqref{eq:worst-case-2} and $g(x_i) = 0$ for all $i \in \N$ that 
\begin{equation*}
  \V(x \mid \mathcal{D}_n) = \min_{\b{u} \in \R^n} \sup_{ \norm[0]{f}_{H(k)} \leq 1} \, \bigg\lvert f(x) - \sum_{i=1}^n f(x_i) u_i \bigg\rvert^2 \geq \min_{\b{u} \in \R^n} \frac{g(x)^2}{\lVert g \rVert_{H(k)}^2} = \frac{g(x)^2}{\lVert g \rVert_{H(k)}^2} > 0
\end{equation*}
for every $n \in \N$.
This is the claim.
\qed

\subsection{Proofs for \Cref{sec:smoothness-matern}} \label{sec:proofs-smoothness}

This section contains the proofs for \Cref{sec:smoothness-matern}.
\\

\noindent \textbf{Proof of \Cref{thm:matern}}
By Lemma~\ref{lemma:sigma-increasing} it suffices to consider the case $\sigma = 0$.
Fix $x \in \R^d$ and let $\delta = \delta(x, X_n) = \min_{x_i \in X_n} \lVert x - x_i \rVert_2$ denote the distance of $x$ to $X_n$.
Define the standard bump function $\varphi$ as
\begin{equation*}
  \varphi(y) = \exp\bigg( -\frac{1}{1-\lVert y \rVert_2^2} \bigg) \: \text{ if } \: \lVert y \rVert_2 < 1 \quad \text{ and } \quad \varphi(y) = 0 \: \text{ otherwise.}
\end{equation*}
Because $\varphi$ is supported on the unit ball and $\lVert x - x_i \rVert_2 \geq \delta$ for every $x_i \in X_n$, the function $g = \varphi((\cdot - x)/\delta)$ vanishes on $X_n$ and $g(x) = \varphi(0) = e^{-1}$.
The Fourier transform of $\varphi$ decays super-polynomially, so that $g \in H(k)$ by the spectral characterisation~\eqref{eq:rkhs-spectral}.
The worst-case interpretation~\eqref{eq:worst-case-2} therefore gives
\begin{equation*}
  \V(x \mid \mathcal{D}_n) = \min_{\b{u} \in \R^n} \sup_{ \norm[0]{f}_{H(k)} \leq 1} \, \bigg\lvert f(x) - \sum_{i=1}^n f(x_i) u_i \bigg\rvert^2 \geq \min_{\b{u} \in \R^n} \frac{g(x)^2}{\lVert g \rVert_{H(k)}^2} = \frac{e^{-2}}{\lVert g \rVert_{H(k)}^2} .
\end{equation*}
We are thus left to estimate the RKHS norm of $g$ from above.
Assume that $\delta \leq 1$.
By the spectral characterisation~\eqref{eq:rkhs-spectral} and the assumed lower bound on the spectral density,
\begin{equation*}
  \begin{split}
  \lVert g \rVert_{H(k)}^2 = \frac{1}{(2\pi)^{d/2}} \int_{\R^d} \frac{\lvert \mathcal{F} g(\omega) \rvert^2}{\mathcal{F}\Phi(\omega)} \dif \omega &\leq \frac{1}{c(2\pi)^{d/2}} \int_{\R^d} \lvert \mathcal{F} g(\omega) \rvert^2 (1 + \lVert \omega \rVert_2^2)^{s} \dif \omega \\
  &= \frac{\delta^{2d}}{ c(2\pi)^{d/2}} \int_{\R^d} \lvert \mathcal{F}\varphi(\delta \omega) \rvert^2 (1 + \lVert \omega \rVert_2^2)^{s} \dif \omega \\
  &= \frac{\delta^{d}}{ c(2\pi)^{d/2}} \int_{\R^d} \lvert \mathcal{F}\varphi(\omega) \rvert^2 \bigg(1 + \frac{\lVert \omega \rVert_2^2}{\delta^2} \bigg)^{s} \dif \omega \\
  &= \frac{\delta^{d}}{ c(2\pi)^{d/2}} \int_{\R^d} \lvert \mathcal{F}\varphi(\omega) \rvert^2 \delta^{-2s} (\delta^2 + \lVert \omega \rVert_2^2 )^{s} \dif \omega \\
  &\leq \frac{\delta^{d-2s}}{ c(2\pi)^{d/2}} \int_{\R^d} \lvert \mathcal{F}\varphi(\omega) \rvert^2 (1 + \lVert \omega \rVert_2^2 )^{s} \dif \omega,
  \end{split}
\end{equation*}
where the integral is finite.
Therefore $\lVert g \rVert_{H(k)}^2 \leq C \delta^{d-2s}$ for a constant $C$ that does not depend on $x$ or $X_n$.
It follows that
\begin{equation*}
  \V(x \mid \mathcal{D}_n) \geq \frac{e^{-2}}{\lVert g \rVert_{H(k)}^2} \geq \frac{1}{C e^2} \delta^{2s - d},
\end{equation*}
which completes the proof.
\qed

\noindent \textbf{Proof of \Cref{thm:analytic}}
Let $f \in H(k)$.
From Proposition~\ref{prop:derivative-repro} and the Cauchy--Schwarz inequality it follows that
\begin{equation} \label{eq:derivative-repro-anal}
  \lvert \mathrm{D}_x^\alpha f(x) \rvert \leq \lVert f \rVert_{H(k)} \lVert \mathrm{D}_x^\alpha k(\cdot, x) \rVert_{H(k)},
\end{equation}
for any $\alpha \in \N_0^d$ and $x \in \R^d$, where the subscript denotes that differentiation is with respect to $x$. Now it follows from the stationarity of $k$ that
\begin{equation*}
  \begin{split}
   \lVert \mathrm{D}_x^\alpha k(\cdot, x) \rVert_{H(k)}^2 = \mathrm{D}_v^\alpha \mathrm{D}_w^\alpha k(v,w) \bigl|_{\substack{v = x \\ w = x}} &= (-1)^{\lvert \alpha \rvert} \mathrm{D}^{2\alpha} \Phi(0) \\
   &= \frac{1}{(2\pi)^{d/2}} \int_{\R^d} \omega^{2\alpha} (\mathcal{F} \Phi)(\omega) \dif \omega \\
   &\leq \frac{C}{(2\pi)^{d/2}} \int_{\R^d} \omega^{2\alpha} \exp(- a \lVert \omega \rVert_2) \dif \omega \\
   &\leq \frac{C}{(2\pi)^{d/2}} \prod_{j=1}^d \int_{\R} \omega_j^{2\alpha_j} \exp(- a \lvert \omega_j \rvert / \sqrt{d}) \dif \omega_j \\
   &= \frac{C}{(2\pi)^{d/2}} \bigg(\frac{\sqrt{d}}{a}\bigg)^{2\lvert \alpha\rvert + d} \prod_{j=1}^d \int_{\R} \omega_j^{2\alpha_j} \exp(- \lvert \omega_j \rvert ) \dif \omega_j \\
   &= \frac{2^d C}{(2\pi)^{d/2}} \bigg(\frac{\sqrt{d}}{a}\bigg)^{2\lvert \alpha\rvert + d} \prod_{j=1}^d (2\alpha_j)! .
  \end{split}
\end{equation*}
A straightforward application of Stirling's approximation in~\eqref{eq:stirling} and~\eqref{eq:derivative-repro-anal} show that $f$ satisfies~\eqref{eq:analytic-derivatives} and is thus real analytic.
Therefore all functions in $H(k)$ are real analytic.

Let $B \subset \R^d$ be an open ball such that $\{x_i\}_{i=1}^\infty$ is dense in $B$.
The span of kernel translates for $x_i \in B$, $W = \operatorname{span} \Set{k^{x_i}}{x_i \in B}$, is dense in $H(k)$ if and only if its orthogonal complement is trivial.
Let $f \in H(k)$ be in the orthogonal complement, which means that $\langle f, k^{x_i} \rangle_{H(k)} = f(x_i) = 0$ for every $x_i \in B$.
Because $\{x_i\}_{i=1}^\infty$ is dense in $B$ and $f$ is real analytic, it follows from \Cref{thm:identity-theorem} that $f = 0$. 
Therefore the orthogonal complement of $W$ is trivial and thus $W$ is dense in $H(k)$.
Let $x \in \R^d$ and $\varepsilon > 0$.
From $W$ being dense in $H(k)$ it follows that there are $N \in \N$, $z_1, \ldots, z_N \in B$, and $\b{v} \in \R^N$ such that
\begin{equation} \label{eq:eps-bound}
  \bigg\lVert k^x - \sum_{j=1}^N v_j k^{z_j} \bigg\rVert_{H(k)} \leq \frac{\varepsilon}{2} .
\end{equation}
Set $\eta = \varepsilon/(2(1+\lVert \b{v} \rVert_1)) > 0$.
Let $L \in \N$.
Because $z \mapsto k^z$ is continuous from $\R^d$ to $H(k)$ and the input points are dense in $B$, we may choose pairwise disjoint open balls $U_1, \ldots, U_N \subset B$ and indices $i_{j,l} \in \N$ for $j \leq N$ and $l \leq L$ such that $x_{i_{j,l}} \in U_j$ and
\begin{equation} \label{eq:eta-bound}
  \lVert k^{z_j} - k^{x_{i_{j,l}}} \rVert_{H(k)} \leq \eta
\end{equation}
for all $j$ and $l$.
Let $n = \max_{j \leq N, l \leq L} i_{j,l}$ and define $\b{u} \in \R^n$ by setting $u_{i_{j,l}} = v_j / L$ for $j \leq N$ and $l \leq L$ and $u_i = 0$ otherwise.
The triangle inequality gives
\begin{equation*}
  \begin{split}
  \bigg\lVert k^x - \sum_{i=1}^n u_i k^{x_i} \bigg\rVert_{H(k)} &\leq \bigg\lVert k^x - \sum_{j=1}^N v_j k^{z_j} \bigg\rVert_{H(k)} + \bigg\lVert \sum_{j=1}^N \bigg( v_j k^{z_j} - \sum_{l=1}^L u_{i_{j,l}} k^{x_{i_{j,l}}} \bigg) \bigg\rVert_{H(k)} \\
  &\leq \bigg\lVert k^x - \sum_{j=1}^N v_j k^{z_j} \bigg\rVert_{H(k)} + \sum_{j=1}^N \frac{\lvert v_j \rvert}{L} \sum_{l=1}^L \lVert k^{z_j} -  k^{x_{i_{j,l}}} \rVert_{H(k)} .
  \end{split}
\end{equation*}
From~\eqref{eq:eps-bound} and~\eqref{eq:eta-bound} it follows that
\begin{equation*}
  \bigg\lVert k^x - \sum_{i=1}^n u_i k^{x_i} \bigg\rVert_{H(k)} \leq \frac{\varepsilon}{2} + \eta \lVert \b{v} \rVert_1 \leq \varepsilon .
\end{equation*}
By Lemma~\ref{lemma:worst-case}, 
\begin{equation*}
  \V_\sigma(x \mid \mathcal{D}_n) \leq \bigg\lVert k^x - \sum_{i=1}^n u_i k^{x_i} \bigg\rVert_{H(k)}^2 + \sigma^2 \lVert \b{u} \rVert_2^2 \leq \varepsilon^2 + \sigma^2 \sum_{j=1}^N \sum_{l=1}^L u_{i_{j,l}}^2 =  \varepsilon^2 + \frac{\sigma^2 \lVert \b{v} \rVert_2^2}{L} .
\end{equation*}
Because $L$ can be chosen arbitrarily large for any given $\varepsilon$ and the variance is non-increasing in $n$ by Lemma~\ref{lemma:worst-case}, we conclude that for every $\varepsilon > 0$ there is $n_\varepsilon \in \N$ such that $\V_\sigma(x \mid \mathcal{D}_n) \leq 2 \varepsilon^2$ for every $n \geq n_\varepsilon$.
Therefore $\V_\sigma(x \mid \mathcal{D}_n) \to 0$ as $n \to \infty$.
Moreover, Dini's theorem implies that the convergence is uniform on every compact set.
\qed

\noindent \textbf{Proof of \Cref{thm:gevrey}}
  The proof is similar to that of \Cref{thm:matern}.
  By Lemma~\ref{lemma:sigma-increasing} it suffices to consider the case $\sigma = 0$.
  Fix $x \in \R^d$ and let $\delta = \delta(x, X_n) = \min_{x_i \in X_n} \lVert x - x_i \rVert_2$ denote the distance of $x$ to $X_n$.
  By Example~1.4.9 and Theorem~1.6.1 in~\cite{Rodino1993}, for any $\tau \in (0, 1)$ there is an infinitely differentiable bump function $\varphi$ supported on the unit ball such that $\varphi(0) = 1$ and $(\mathcal{F}\varphi)(\omega) = O(\exp(- b \lVert \omega \rVert_2^{\tau}))$ for some $b > 0$.
  The function $g = \varphi((\cdot - x)/\delta)$ vanishes on $X_n$ and $g(x) = \varphi(0) = 1$.
  It immediately follows from the spectral characterisation~\eqref{eq:rkhs-spectral} that $g \in H(k)$ if $k_\eta$ is a Gevrey kernel of order $\eta < \tau$.
  Just as in the proof of \Cref{thm:matern} we obtain
  \begin{equation*}
    \V(x \mid \mathcal{D}_n) \geq \frac{1}{\lVert g \rVert_{H(k_\eta)}^2}
  \end{equation*}
  from~\eqref{eq:worst-case-2}.
  It remains to estimate $\lVert g \rVert_{H(k_\eta)}$.
  There is a positive constant $C_1$ independent of $\delta$ such that
  \begin{equation*}
    \begin{split}
    \lVert g \rVert_{H(k_\eta)}^2 = \frac{1}{(2\pi)^{d/2}} \int_{\R^d} \frac{\lvert (\mathcal{F}g)(\omega)\rvert^2}{(\mathcal{F} \Phi_\eta)(\omega)} \dif \omega &= \frac{\delta^{2d}}{C_\eta (2\pi)^{d/2}} \int_{\R^d} \lvert (\mathcal{F}\varphi)(\delta \omega) \rvert^2 \exp(\rho \lVert \omega \rVert_2^{\eta}) \dif \omega \\
    &= \frac{\delta^{d}}{C_\eta (2\pi)^{d/2}} \int_{\R^d} \lvert (\mathcal{F}\varphi)(\omega) \rvert^2 \exp(\rho \delta^{-\eta} \lVert \omega \rVert_2^{\eta}) \dif \omega \\
    &\leq \frac{C_1}{C_\eta (2\pi)^{d/2}} \delta^d \int_{\R^d} \exp(\rho \delta^{-\eta} \lVert \omega \rVert_2^{\eta} - 2 b \lVert \omega \rVert_2^{\tau}) \dif \omega .
    \end{split}
  \end{equation*}
  The exponent in the integral has the maximum
  \begin{equation*}
    2b( \tau/\eta - 1 ) \bigg(\frac{\eta \rho}{2 \tau b}\bigg)^{1/(1-\eta/\tau)} \delta^{-1/(1/\eta-1/\tau)} \eqqcolon \theta \delta^{-1/(1/\eta-1/\tau)}
  \end{equation*}
  and is less than $-b \lVert \omega \rVert_2^{\tau}$ when $\lVert \omega \rVert_2 > r(\delta) \coloneqq (\rho b^{-1})^{1/(\tau - \eta)} \delta^{-1/(\tau/\eta-1)}$.
  Therefore there are positive constants $C_2$ and $C_3$ that do not depend on $\delta$ such that
  \begin{equation*}
    \begin{split}
    \lVert g \rVert_{H(k_\eta)}^2 &\leq C_2 \delta^d \bigg( \int_{\lVert \omega \rVert_2 \leq r(\delta)} \exp(\theta \delta^{-1/(1/\eta-1/\tau)}) \dif \omega + \int_{\lVert \omega \rVert_2 > r(\delta)} \exp(-b \lVert \omega \rVert_2^{\tau}) \dif \omega \bigg) \\
    &\leq C_3 \delta^{(1/\eta - 2/\tau)d/(1/\eta-1/\tau)} \exp(\theta \delta^{-1/(1/\eta-1/\tau)}) 
    \end{split}
  \end{equation*}
  when $\delta$ is sufficiently small.
  Here we used the fact that the volume of an $r$-ball in $\R^d$ is proportional to $r^d$.
  The claim now follows.
\qed

\section*{Acknowledgements}

TK was supported by the Research Council of Finland grants 338567 (``Scalable, adaptive and reliable probabilistic integration''), 359183 (``Flagship of Advanced Mathematics for Sensing, Imaging and Modelling''), and 368086 (``Inference and approximation under misspecification'').
TK also acknowledges the research environment provided by ELLIS Institute Finland.
CJO was supported by a Philip Leverhulme Prize (PLP-2023-004). 
The origins of this article are in the positive response that TK received to talks given on the topic at the \emph{Prob Num 2022} workshop in London in March 2022 and at the \emph{ELISE Theory Workshop on Machine Learning Fundamentals} at EURECOM in September 2022.
We are grateful to Anatoly Zhigljavsky for discussions that helped to kindle our interest in inverse multiquadrics, to Luc Pronzato for introducing us to literature on the no-empty-ball property, and to Vesa Kaarnioja for ample discussion on Gevrey spaces and kernels.

\bibliography{references}

\end{document}